\documentclass[letterpaper, 10 pt, conference]{ieeeconf}  % Comment this line out if you need a4paper

\IEEEoverridecommandlockouts                              % This command is only needed if
\usepackage{amsmath}
\usepackage{amssymb}
\usepackage{mathtools}
\usepackage{enumerate}

\usepackage{afterpage}

\newtheorem{theorem}{Theorem}[section]

\newtheorem{Assumption}{Assumption}[section]

\newtheorem{lemma}[theorem]{Lemma}

\usepackage{subcaption}
\usepackage{float}
\usepackage{xcolor}
\usepackage{bm}
\usepackage{hyperref}

\allowdisplaybreaks
\newcommand{\norm}[1]{\left\lVert#1\right\rVert}
\newcommand{\abs}[1]{\left|#1\right|}

\newcommand{\e}{e}

\newcommand{\Mw}{M_{w}}
\newcommand{\hMw}{\widehat{\mathbf{M}}_{w}}

\newcommand{\htMw}{\widehat{\widetilde{\mathbf{M}}}_w}
\newcommand{\tMw}{\widetilde{M}_w}

\newcommand{\bu}{\mathbf{u}}
\newcommand{\bq}{\mathbf{q}}
\newcommand{\bx}{\mathbf{x}}
\newcommand{\bX}{\mathbf{X}}
\newcommand{\by}{\mathbf{y}}
\newcommand{\bz}{\mathbf{z}}
\newcommand{\bw}{\mathbf{w}}
\newcommand{\bv}{\mathbf{v}}
\newcommand{\bchi}{\bm{\chi}}
\newcommand{\bxi}{\bm{\xi}}
\newcommand{\bXi}{\bm{\Xi}}
\newcommand{\bHhat}{\mathbf{\widehat{H}}}
\newcommand{\bAhat}{\mathbf{\widehat{A}}}
\newcommand{\bBhat}{\mathbf{\widehat{B}}}
\newcommand{\bChat}{\mathbf{\widehat{C}}}

\newcommand{\sat}{\mathrm{sat}}

\DeclarePairedDelimiter\set\{\}

\title{\LARGE \bf
A finite-sample bound 
for identifying
partially observed linear switched systems from a single trajectory}

\author{Dániel Rácz$^{1}$,
 Mihály Petreczky$^{2}$ and Bálint Daróczy$^{3}$% <-this % stops a space
\thanks{Supported by the European Union project RRF-2.3.1-21-2022-00004 within
the framework of the A.I. National Laboratory Program.
}% <-this % stops a space
\thanks{$^{1}$HUN-REN SZTAKI and ELTE, Budapest, Hungary
        {\tt\small racz.daniel@sztaki.hun-ren.hu}}%
\thanks{$^{2}$
Univ. Lille, CNRS, Centrale Lille, UMR 9189 CRIStAL, F-59000 Lille, France}
\thanks{$^{3}$HUN-REN SZTAKI, Budapest, Hungary}
}

\begin{document}

\maketitle
\thispagestyle{empty}
\pagestyle{empty}

%%%%%%%%%%%%%%%%%%%%%%%%%%%%%%%%%%%%%%%%%%%%%%%%%%%%%%%%%%%%%%%%%%%%%%%%%%%%%%%%
\begin{abstract}
We derive a finite-sample probabilistic bound on the parameter estimation error of a system identification algorithm for Linear Switched Systems. The algorithm estimates Markov parameters from a single trajectory and applies a variant of the Ho-Kalman algorithm to recover the system matrices.
%of inpuoutput, and switching signals, sampled from a stochastic process. 
%The system matrices are recovered using the Ho-Kalman algorithm, Markov parameters or the empirical estimates. 
%The main result of this paper is an estimation error bound on the difference between these two approaches for stable systems.
Our bound guarantees statistical consistency under the assumption that the true system exhibits quadratic stability. The proof leverages the theory of weakly dependent processes.
To the best of our knowledge, this is the first finite-sample bound for this algorithm in the single-trajectory setting.
\end{abstract}

%%%%%%%%%%%%%%%%%%%%%%%%%%%%%%%%%%%%%%%%%%%%%%%%%%%%%%%%%%%%%%%%%%%%%%%%%%%%%%%%

\section{Introduction}
\label{sec:introduction}
In this paper we consider the problem of identifying 
a \emph{Linear Switched System (LSS)}
\begin{align}
\label{eq:1}
\Sigma &
\begin{cases}
    \bx(t+1) = A_{\bq(t)} \bx(t) + B_{\bq(t)} \bu(t) + \bw(t)\\
    \by(t) = C \bx(t) + D \bu(t)+\bv(t)
\end{cases}
\end{align}
For simplicity, we consider MISO systems, i.e. \textcolor{black}{the input} $\bu(t) \in \mathbb{R}^m$, \textcolor{black}{the output} $\by(t) \in \mathbb{R}$ \textcolor{black}{and the state} $\bx(t) \in \mathbb{R}^n$, where $n$ is the dimension of the state space, \textcolor{black}{$m$ is the dimension of the input}, \textcolor{black}{the switching signal} $\bq(t) \in \{1, ..., n_Q\} = Q$, which is a finite set, \textcolor{black}{and $\bw(t)$ and $\bv(t)$ are the noise vectors.}
We would like to estimate the matrices $(\{A_q, B_q\}_{q \in Q}, C, D)$ (up to similarity) from a sample $\mathcal{S}_N=\{(y_t, u_t, q_t)\}_{t=0}^N$ of
$(\by(t),\bu(t),\bq(t))$
and derive finite-sample bounds on the parameter estimation error.

\textbf{Contribution.} We present a finite-sample probabilistic
bound on the
error between the true parameters and those returned by
a particular subspace identification algorithm, described
in \cite{CoxLPVSS,Rouphael2024}. 
The latter algorithm relies on estimating the
Markov-parameters of the LSS \cite{Petreczky2015}, and applying Ho-Kalman algorithm to
the arising Hankel matrix.
The bound on the estimation error of the system matrices is $O(1/\sqrt{N})$, therefore it converges to $0$ as $N \rightarrow \infty$.
\textcolor{black}{We assume bounded, i.i.d. inputs,
noises, and switching processes, as well as the stability of the true system.}  These simplifying assumptions ensure persistence of excitation \cite{PETRECZKY2023101308},  making them a natural starting point for deriving the first finite-sample bound.
We consider scalar outputs for simplicity, but the results can be extended to the MIMO case in a straightforward manner, \textcolor{black}{see \cite[Appendix B]{racz2025finite}}.

Similarly to bounds of this type, the bound is polynomial in the underlying system's state-space dimension, and decreases with the Lyapunov exponent of the underlying system. 
Unlike much of the related work, we do not assume any lower bound on the number of  Markov parameters needed to be estimated for a given trajectory length. On the downside, we need the boundedness of inputs and noise. In contrast to most of the existing work, instead of concentration inequalities for sub-Gaussian random matrices, we use concentration inequalities for weakly dependent processes \cite{alquier2012}.
%sample bounds for linear systems, we do not have the condition that the trajectory needs to be long enough in order to estilar
%state is zero, in fact, we assume that the training data already represents the stationary regime of the system, and we do not require the data to be long enough to suppress the effect 

%In order to motivate the contribution of the paper, we first define the system identification problem of LSSs. 
%To this end, 
%    let us define the \emph{deterministic behavior} $\mathcal{B}$ of 
%    a tuple $S=(\{A_q,B_q\}_{q \in Q},C,D)$ of matrices
%    as the set of all deterministic signals $(y,u,q)$,
%    all defined on $\mathbb{Z}$, taking values in 
%    $\mathbb{R}$ and $Q$ respectively,
%     such that 
%    there exists a state trajectory $x$ and noise realizations $v,w$ satisfying:
%    \begin{equation}
%	\label{eq:aslpv:def}
%            \begin{aligned}
%                 &  x(t+1)=
%               %\sum_{i=1}^{\pdim} 
%               A_{q(t)}x(t)+B_{q(t)}u(t)+w(t) \\
%                 & y(t)=Cx(t)+Du(t)+v(t)
%		\end{aligned}
%    \end{equation}
  %That is, $\mathcal{S}$ can be viewed   as a stochastic version of \eqref{eq:1}. 
 % In particular, all samples paths of $(\y,\bu,\bq)$ are  elements of the deterministic behavior of $S$. 
  %The system identification problem is as follows: 

\textbf{Motivation.}
  Finding finite-sample bounds on the system matrices allows us to evaluate the performance of the estimated model in terms of their predictive power for any inputs and switching sequences, and
  their use for control design.
  %and not only for those which were used for training. 
  %If we find the matrices $(\{A_q,B_q\},C,D)$ up to
  %a similarity transform from $\mathcal{D}$, then 
  %we obtain a model, input-output behavior of which will
  %be the same as that of \eqref{eq:aslpv:def}
  %for \emph{any input and switching signal}.
  %That is, assuming that the data used for identification
%experiment are sampled from a stochastic process does not mean that the identified model is valid only for the stochastic processes from which the data
%were sampled. In fact, stochasticity of the inputs and
%  switching signals can be seen as a
%  for persistence of excitation condition for the data.
%  This is further demonstrated by \cite{PETRECZKY2023101308}, which
%  shows that white noise input and binary white noise
%  switching signal are persistently exciting. 

\textbf{Related work: system identification.}
Identification of switched and jump-Markov systems is an active research area, see e.g. \cite{LauerBook}. %, and the references therein. However, 
Most of the literature assumes that the switching signal is unobserved,
making the learning problem more challenging and leading to 
%\textcolor{black}{thus is unable to} 
the absence of consistency results. 
In contrast, we assume that the switching signal is observed,
which is a more restrictive assumption, but still relevant for many application, and 
we provide a finite-sample bound which also proves consistency.
%That is, unobserved switching makes the learning problem more challenging, in particular it also leads to lack of consistency results for state-space representations.
Furthermore, there are consistent subspace
    identification algorithms, e.g. \cite{PETRECZKY2023101308} for noiseless LSSs,
    or in \cite{Rouphael2024} for noisy LSS. Both algorithms assume that the switching signal is known. To our knowledge, there are no consistency results without some knowledge about the switching signal. 
    
    %In \cite{sarkar2019nonparametric} a consistent identification algorithm for noisy LSSs was presented,
    %but the noise gain matrix and the noise covariance 
    %were not estimated, and the identification was based on several Mi.i.d. time series. 
    %For jump-Markov systems there is a long history of applying
    %expectation-maximization based algorithms (EM). However, in general it
    %is unclear to which extent these methods are consistent, and more,
    %importantly, if the input-output behavior of the estimated system
    %is close to the true one for inputs and switching signals which are
    %different from the ones used for identification.
    %Moreover, most of the methods assume a certain parameterization
    %of the state-space representation, with the exception of \cite{bako2009line}. However, \cite{bako2009line} assumes only output noise and does not propose consistency guarantees.
    
    LSSs can be viewed as a subclass of linear parameter-varying (LPV)
    systems, if the switching signal is viewed as a discrete scheduling. There is a wealth
    of literature on system identification for LPV systems, including subspace methods \cite{CoxTothSubspace, Toth2010SpringerBook}, %and the references therein, 
    but there are no finite-sample bounds for the %case
    %of training data being a 
    single time series setup.
     %However, most of the literature LPV susbspace identification provides no proofs of consistency of subspace  methods, and most of them do not estimate the noise gains and noise variance. 
     %Stochastic
%     realization theory of LPV systems was invetsigated in \cite{MejariLPVS2019,CoxLPVSS}.
 %    In \cite{CoxLPVSS} the existence of LPV systems in innovation form
 %    was investigated, 
     %but the resulting state-space representation 
     %had a dynamical dependence on scheduling variables, i.e., 
 %    but the noise
 %    gain matrices of the obtained %system depended on the current and %past scheduling. 
 %    Moreover, \cite{CoxLPVSS} did not address minimality and uniqueness 
 %    of systems in innovation form. 
   %M  Unlike the identification algorithm of this  paper, 
     %However, existing subspace
     %methods for LSSs, \cite{CoxTothSubspace}, 
     %are not proven to be  consistent  except \cite{MejariLPVS2019}. 
     %for consistency or for returning a minimal system innovation form. 

\textbf{Related work: finite-sample bounds.}
\textcolor{black}{Probably Approximatelly Correct (PAC, see \cite{shalev2014understanding} for a general overview)} bounds for sub-classes of
dynamical systems exist in autoregressive \cite{alquier2012, MASSUCCI202155}
and state-space form \cite{eringisRenyi,eringis2023pacbayes}.
However, the cited papers relate the true loss with the empirical loss,
and except \cite{eringisRenyi} which studies only linear systems,
the cited papers do not address the parameter estimation error. 
%The cited papers assumed bounded inputs, %\cite{vidyasagar2006learning} bounded loss, and \cite{campi2002finite} assumed bounded inputs and noise and finite horizon prediction. 
%Moreover, \cite{vidyasagar2006learning} restricts attention to a small subset of linear systems in input-output form, and the error bound of \cite{campi2002finite} is exponential in the number of parameters.
%In  \cite{alquier2013prediction} 
%PAC-Bayesian bounds for auto-regressive models without exogenous inputs were considered, and the variables were either assumed to be bounded or the loss function was assumed to be Lipschitz. In \cite{eringisRenyi} PAC-Bayesian bounds  for linear state-space models were developed.
%probabilistic bound depends on the confidence level as  $O(\frac{1}{\delta})$ instead of the standard 
%$O(\ln \frac{1}{\delta})$. 
%Finally, \cite{eringis2023pacbayesRNN} presents a PAC-Bayesian bound for RNNs.

%There has been a surge of publications on
%finite-sample bound for various
%learning algorithms applied to dynamical systems 
%%for an overview see 
%\cite{FiniteSampleOverview}. 
%Many papers concentrate either on autoregressive
%models or on models with full state observation,
%including switched models, e.g.,
%\cite{MASSUCCI202155,sayedana2024strong,shi2022finite,ziemann2022single}.
%A few papers propose a finite-sample bound on linear dynamical systems for identification via estimating the Markov parameters and applying Ho-Kalman, e.g. \cite{oymak2021revisiting}.
%

Recent works on finite-sample bounds on learning algorithms for dynamical systems \cite{FiniteSampleOverview} focus on autoregressive models or full-state observation models, including switched models \cite{MASSUCCI202155, sayedana2024strong, shi2022finite, ziemann2022single}. A few papers propose finite-sample bounds for identifying linear dynamical systems \cite{simchowitz2018learning} by estimating Markov parameters and applying Ho-Kalman \cite{oymak2021revisiting}.
%There are somewhat fewer papers on several finite-sample bounds for learning linear dynamical systems were derived for system identification algorithms based on estimating Markov parameters and  applying Ho-Kalman algorithm, e.g., \cite{oymak2021revisiting}.
%For switched systems in autoregressive form or with full state observation, finite-sample bounds
%it was extended to nonlinear models too %\cite{OzayCDC2022sattar2022non,blanke2023flex,foster2020learning,mania2022active,,shi2022finite,Roy_Balasubramanian_Erdogdu_2021,Ziemann_Sandberg_Matni_2022,Ziemann_Tu_2022,Li_Ildiz_Papailiopoulos_Oymak_2023}
%\cite{sattar2025finitesampleidentificationpartially} proposed a finite-sample bound for the estimation error of Markov parameters for 
%bilinear systems. While the definition of Markov parameters for bilinear systems is similar to that of the LSS, the problem formulation is different and the results cannot be carried over to LSS directly. 
\cite{sattar2025finitesampleidentificationpartially} proposed a finite-sample bound for estimating Markov parameters for bilinear systems. However,
\textcolor{black}{all these} results are not transferable to LSSs, due to differences between the system classes. 
%Though their definition resembles that in LSS, the problem formulation differs, making the results non-transferable to the case of LSS.

In \cite{sarkar2019nonparametric} a similar problem was considered, but the training data consisted of several independently sampled time series. 
%The main difference w.r.t. \cite{sarkar2019nonparametric} is
%that we assume that the training data is a single time series. 
Another difference is that we consider a version of the Ho-Kalman 
algorithm with a basis selection. The latter has been used for
subspace identification \cite{CoxLPVSS,CoxTothSubspace}, and it
is computationally more efficient.

To sum it up, the main novelty of this paper is proposing a
finite-sample bound for estimating Linear Switched Systems with \textbf{(1)} \emph{partial observations} from a \textbf{(2)} \emph{single time series}, using the theory of \textbf{(3)} \emph{weakly-dependent} processes \cite{alquier2012}.
%
%\textcolor{blue}{TODO: mi a noveltynk ezekhez ke
%
The paper is organized as follows. In Section~\ref{sec:preliminaries}, we define our problem. Next, we state our main result in Section~\ref{sec:main}, conclude our findings in Section~\ref{sec:conc}  
followed by the proofs of the Lemmas in Section~\ref{sec:proofs}. 

\section{Problem setup}
\label{sec:preliminaries}

\textbf{Notation.}
For any vector $v \in \mathbb{R}^n$ we denote the $i$-th component of $v$ by $v[i]$, while $v[k:m]$ is the component of $v$ between the indices $k$ and $m$.
A word $w = w_1 \ldots w_k$ refers to the tuple $(w_1, \ldots w_k)$. The length of the word is denoted by $|w|$. For any words $v$ and $w$, $vw$ denotes their concatenation. \textcolor{black}{The empty word is denoted by $\epsilon$ and $Q^*$ denotes the Kleene star of $Q$.} We use \textbf{boldface} characters for probabilistic quantities. 
For any matrix $A$, $\sigma_n(A)$ denotes its smallest singular value, while $\lambda_1(A)$ and $\lambda_n(A)$ denote the eigenvalues with the largest and smallest absolute values, respectively.
\textcolor{black}{$I_n \in \mathbb{R}^{n \times n}$ denotes the identity matrix, $0$ denotes the squared zero matrix and $|$ denotes the concatenation of matrices.}
We denote by $\norm{\cdot}_p$ the $p$-norm $p=1,2,\infty$ on the Euclidean space and \textcolor{black}{$\norm{\cdot}_F$ denotes the Frobenius norm}.
All the random variables are understood over the probability space $(\Omega,\mathcal{F},\mathbb{P})$, $\mathbb{E}$ denotes the expectation and
all the stochastic processes are over the discrete time axis $\mathbb{Z}$. 

%\textcolor{black}{\textbf{TODO: szoveg arrol, hogy akkor miert bold itt az u meg x. Lehet, hogy itt meg nem kellen bold, csak kesobb, mert alatta megmondjuk, hogy hany dimenzios es akkor az mer ugy nem leszjo }}

%Additionally, we use the notations $p_c = \mathbb{P}\left[\bq_t = c\right]$, $p_w = p_{w_1} \ldots p_{w_k}$ for a word $w = w_1\ldots w_k$ and $p_{\emptyset} = 1$.

%\subsection{Assumptions}

%We have a sample $S = \set{y_t, u_t, q_t}_{t=0}^N$ sampled from a stochastic process of the form \eqref{eq:1}
%where
\textbf{Preliminaries on switched systems}.
We make the following assumptions on the LSS parameters \textcolor{black}{defined in \eqref{eq:1}}. % of which are to be estimated. 
\begin{Assumption}
\label{ass:main}
We assume the following on the processes and matrices involved in the
definition of \eqref{eq:1}:
\begin{enumerate}[a)]
\item  ${\bq(t)}$ is i.i.d., $\mathbb{P}({\bq(t)} = q) = p_q > 0$,
\item 
     \textcolor{black}{$\mathbf{r}(t)=[\bu^T(t), \bw^T(t), \bv^T(t)]^T$} is a zero mean i.i.d. process with \textcolor{black}{$\Sigma_u=\mathbb{E}[{\bu(t) \bu(t)^T}]$ being the covariance matrix of $\bu$},
     and for 
     $t \in \mathbb{Z}$, 
     $\norm{\mathbf{r}(t)}_{\infty} \le K_u$,  and
     %the $\sigma$-algebras
     %generated by $\{\bq(s)\}_{s=-\infty}^{\infty}$ and 
     %$\{\mathbf{r}(s)\}_{s=-\infty}^{\infty}$ are %independent, and 
     $\mathbf{r}(t)$ and $\bx(t)$ are independent, 
     and also $\{\bw^T(t),\bv^T(t)\}$  and  $\bu(t)$ are independent.
 %\textcolor{black}{Both $\bw$ and $\bv$ are i.i.d., independent from $\bu$, and bounded? $\norm{\bw}_\infty \leq L_\bw$ and $|\bv| \leq L_\bv$ }
%\ \textcolor{black}{\textbf{b) $q_t$ is not iid but Markov??}}
%\textcolor{cyan}{\item ${\bu(t)} \in \mathbb{R}^m$ is a vector where the elements are i.i.d. with zero mean, $\Sigma_u=\mathbb{E}[{\bu(t) \bu(t)^T}]$ and
%\textcolor{black}{$\norm{\bu(t)}_\infty < L_\bu$} with probability $1$ and let $K_u = \max\{L_\bu, L_\bw, L_\bv\}.$} 
%\textcolor{black}{ennek nem kellene mindig igaznak lennie?},}
\item
${\set{\mathbf{r}(s)}_{s \in \mathbb{Z}}}$ and ${\set{\bq(s)}_{s \in \mathbb{Z}}}$ are independent,
%\item ${\bx(t)}, {\by(t)}$ are defined for all $ t \in \mathbb{Z}$,
\item the noiseless LSS,
 %\begin{equation}
 %\label{LSS:det}
  \( x_{t+1}=A_{q_t}x_t+ B_{q_t}u_t, ~ y_t=Cx_t+Du_t \)
  % \end{equation}
  has minimal dimension in the sense of \cite{Petreczky2015}, i.e. it is span-reachable from zero and observable, %\cite{Petreczky2015}, 
\item there exists a positive definite matrix
$P \succ 0$ and $\textcolor{black}{\gamma \in (0,1)}$
such that $\forall q \in Q$,
$A_q^TPA_q \prec \textcolor{black}{\frac{\gamma^2}{n^2_{Q}}} P$.
%\item \textcolor{black}{TODO: Bounded input? $u_t < K_u$}
%\item \textcolor{black}{TODO: Tobb helyen hasznalom, hogy $y_t \leq K_y$. ezzel mi legyen?}
%\item \textcolor{black}{Bounded output (as a result of stability): $y_t \leq K_y$}
\end{enumerate}
\end{Assumption}

Assumption a) is somewhat restrictive, but reasonable when switching is treated as an external input, as it is common in LSS applications.
Assumption b) requires the inputs and the noise to be i.i.d., bounded, and that the noise is independent of state/input. Independence is standard; boundedness is typical in robust control, though it is restrictive. I.i.d. inputs are common in finite-sample bounds, but restrictive for identification.
Assumption c) assumes the switching to be independent of inputs/noise, realistic for external, but excludes state-dependent switching.
Assumption d) ensures well-posed parameter estimation: $A_q,B_q,C$ are unique up to a coordinate change.
Assumption e) implies quadratic stability, ensuring mean-square stability \cite{CostaBook}, similar to \cite{sarkar2019nonparametric}.

Assumption e) implies that i.i.d. $\bq$ and $\bu$ act as the persistence of excitation: once system matrices are identified (up to a basis change), the model is input-output equivalent to the true system for any input or switching sequence, similar to the standard principles in system identification.

%These assumptions allow us to recall Markov parameters for LSS and present a formula based on the covariances of the input. 
Next, we recall from \cite{CoxLPVSS,Petreczky2015,sarkar2019nonparametric} the definition of Markov parameters and the covariance formula for them. 

%\textcolor{black}{Furthermore, the assumption on the noise allows us to consider noiseless systems during the rest of the paper, as we can consider $\bu'(t) = [\bu(t), \bw(t), \bv(t)]^T \in \mathbb{R}^{m + n + 1}$ and modify $B_q$ and $D$ accordingly.} 

For any word $w \in Q^{*}$
%Let $w$ denote the word
define $A_w$ as \textcolor{black}{the matrix product} $A_w=A_{w_k} \cdot \ldots \cdot A_{w_1}$  if
$w = w_1 \ldots w_k$, where $w_i \in Q$, $i=1,\ldots,k$, and let
$A_{\epsilon} = I_n$.
Define the \emph{Markov parameters} $\{M_v\}_{v \in Q^{*}}$ of $\Sigma$ for any $i \in Q$ and $w \in Q^{*}$ as
\begin{align}
\label{eq:M}
M_{iw} = CA_wB_i,  ~ M_{\epsilon}=D.
              %C A_{w_k} \cdot \ldots \cdot A_{w_2} B_{w_1} & k > 1 \\
              %CB_{w_1} & k=1
          %\end{array}\right.
\end{align}
It follows from realization theory of LSSs \cite{Petreczky2015} that the Markov parameters $\{M_w\}_{w \in Q^{*}}$ determine
the matrices $\{A_q,B_q\}_{q \in Q}, C$ up to a change of basis, due to minimality. % of \eqref{LSS:det}.
In fact, there is a Ho-Kalman-like realization algorithm for computing the system matrices from the Markov parameters.
In turn, the Markov parameters can be computed as covariances of the input and output as follows. 
For any vector valued, discrete-time stochastic process
$\mathbf{h}$ define
\begin{align*}
&\bz_w^{\mathbf{h}}(t) = \mathbf{h}(\textcolor{black}{t-k})\bchi_w(t-1) \\
&\bchi_{c}(t) := \begin{cases}
1 &  \text{if } \bq(t) = c \in Q \\
0 & \text{otherwise}
\end{cases} 
\end{align*}
and \textcolor{black}{$\bchi_w(t) = \bchi_{w_1}(t - k+1 )\cdot \ldots \cdot \bchi_{w_k}(t)$} % \quad \text{if } w=q(1) \ldots q(k)\\
for $Q^* \ni w=$\\$= w_1 \ldots w_k$, $k > 1$, and set $\bchi_{\epsilon}(t)=1$, 
\textcolor{black}{$\bz_{\epsilon}^{\mathbf{h}}(t)=\mathbf{h}(t)$}.

\begin{lemma}
\label{lem:stab:stat}
Let $B'_q=[B_q \mid I_n \mid 0 ]$,  and 
for all $v \in Q^{*}$,
$M'_{qv}=CA_vB'_q$ and 
\textcolor{black}{$M'_{\epsilon}=[ D, 0 ] \in \mathbb{R}^{1 \times m+1}$}, and 
\\$m' = m + n + 1$.
Under Assumption \ref{ass:main} the followings hold:
\begin{enumerate}[a)]
\item $\exists~ \gamma \in (0,1)$, $K_M > 0$ such that
       $\norm{M'_w}_2 < K_M \frac{\gamma^{|w|}}{n_Q^{|w|}}$.
%\textcolor{black}{ez miert kettes norma? skalarra csinalunk mindent, persze ebbol az kovetkezik, de nem konzisztens}
\item
There exist unique stationary state
and output trajectories $\bx$ and $\by$ which satisfy \eqref{eq:1} and
\begin{align*}
%\bx(t) = \sum_{\substack{v \in Q^{*}\\ q \in Q}}  A_vB_q \bz_{v}^\bu(t), ~
& \bx(t) = \sum_{(v, q) \in Q^{*}\times Q}  A_vB'_q \bz_{v}^{\mathbf{r}}(t), ~ \by(t) = \sum_{v \in Q^{*}}  M'_v \bz_{v}^{\mathbf{r}}(t)
\end{align*}
%$\bu'=[\bu^T, \bw^T, \bv^T] ^T$,  Let 
where the infinite sum converges in
the mean square sense,
and for all $w = w_1 \ldots w_k \in Q^{*}$
\begin{align}
\label{eq:Mw}
M_{w} = \mathbb{E}\left[{{\by(t)} {\bz}^{{\bu}}_{w} {(t)} }\right]^T\Sigma_u^{-1}p_w^{-1}
\end{align}
where  $p_{\epsilon}=1$ and $p_{w}=\prod_{i=1}^{k} p_{w_i}$
for $w_i \in Q, k > 0$.% $w_1,\ldots, w_k \in Q$, $k > 0$.
\item $\by(t)$ is bounded, i.e. 
$|\by(t)| < K_y:=\frac{\textcolor{black}{\sqrt{m'}}K_M K_u}{1-\gamma}$.
\end{enumerate}
\end{lemma}

\textcolor{black}{The proof can be found in \cite[Appendix A]{racz2025finite}. }
%We call
%M$\gamma$ the Lyapunov exponent of %\eqref{eq:1}.
%We know \textcolor{black}{\textbf{TODO: honnan? }} that
%where
%\begin{align}
%z_w^u(t) = u_{t-k-1}\frac{\chi_w(t-1)}{\sqrt{p_w}}
%\end{align}
%\begin{align*}
% &E[   u(t-k-1)\chi_w(t-1)u(t-l-1)\chi_v(t-1)]=\\
% &= E [u(t-k-1)u(t-l-1)] E\chi_w(t-1)\chi_v(t-1)=\\
% &\left\{\begin{array}{rl} 0 & k \ne l, w \ne v \\
%        \sigma^2 p_w & w=v,k=l
%    \end{array}\right.
%\end{align*}
%\textbf{Nem normalt. melyik legyen?}
%\begin{align}
%\label{eq:ME}
%M_{w} = \mathbb{E}\left[{{y_t} {z}^{{u}}_{w} {(t)} }\right],
%\end{align}
%
%and $A_{w}=A_{w_k}\ldots A_{w_1}$, then
%\textbf{normalt:}
%\begin{align}
%y_t = \sum_{v \in Q^+} \sqrt{p_{v}} M_v z_{v}^u(t) + Du_t,
%\end{align}
%where $Q^+$ is the set of all non-empty words over $Q$ and \textcolor{black}{the convergence is in the mean-squared sense}. \textcolor{black}{\textbf{TODO: cite ez miert igaz}}
%Next, we will propose a learning algorithm based on a version of Ho-Kalman algorithm for estimating
%$A_q,B_q,C,D$.% from data.

%Under \textcolor{black}{\textbf{pontosan melyik assumptionok? a)? b)?}} we can estimate the quantities $M$ with $\widehat{M}$ by the Ho-Kalman algorithm \textbf{TODO: cite}. We can then estimate the system matrices by $\widehat{A}_q, \widehat{B}_q, \widehat{C}$ and $\widehat{D}$.

\textbf{Reduced basis Ho-Kalman algorithm}. 
%As it was mentioned above, from the knowledge of the Markov parameters, an isomorphic copy of $\Sigma_q$ can be recovered using the following version of the Ho-Kalman algorithm \textcolor{black}{\cite{CoxLPVSS,PETRECZKY2023101308,petreczky2010spaces},
%which, unlike naive extensions of Ho-Kalman algorithm to LSSs, does not grow exponentially with the size of the state-space. The main idea is to fix rows and columns of the naive Hankel-matrix \cite{Petreczky2015} which
%form a basis. }
%As mentioned earlier, 
An isomorphic copy of \(\Sigma\)  can be recovered from the Markov parameters using the reduced basis Ho-Kalman algorithm \cite{CoxLPVSS,PETRECZKY2023101308,petreczky2010spaces}, which, unlike naive extensions of Ho-Kalman to LSSs, avoids any exponential growth of the Hankel-matrix. %with state-space size 
%by selecting specific rows and columns of the naive Hankel matrix \cite{Petreczky2015} that form a basis.  

%Let us define the notion of (nice) selection \citep{bastug2016model} or \citep{petreczky2010spaces}. Let $\alpha \subseteq Q^*, \beta \subset Q^* \times Q$, with $\len{a} = \len{\beta} = n$, such that $\forall v \in \alpha$, $|v| \leq n-1$, and $\forall (w,q) \in \beta$, $\len{w} \leq n-1$. Here, $Q^*$ is the set of all words (sequences) over $Q$, including the empty sequence $\eps$. $\len{w}$ is the length of the sequence $w$ and $\len{\eps} = 0$.
%More precisely, consider the set of words with $n$ elements $\alpha = \{\eta_1, \ldots, \eta_n\} \subseteq Q^*$ along with
%another set of $n$ elements $\beta = \{(\mu_1,q_1,l_1), \ldots, (\mu_n q_n,l_n)\} \subset Q^* \times Q \times \{1,\ldots,m\}$ such that
%$|\eta_i| \le n-1$, 
%$|\mu_i| \le n-1$,
%$i=1,\ldots,n$.
%We call such sets $n$-row and $n$-column selection respectively.
%where $\beta_i$ has the
%form $\beta_i = \beta^0_iq_i$ such that $\beta^0_i \in Q^*$ and $q_i \in Q$. 
Consider the sets \(\alpha = \{\eta_1, \dots, \eta_n\} \subseteq Q^*\) and
\\ \(\beta = \{(\mu_1, q_1, l_1), \dots, (\mu_n, q_n, l_n)\} \subset Q^* \times Q \times \{1, \dots, m\}\), where \(|\eta_i|, |\mu_i| \leq n-1\) for \(i = 1, \dots, n\). These are called \(n\)-row and \(n\)-column selections, respectively.
Intuitively, \textcolor{black}{they} represent a selection of rows and columns of the standard Hankel-matrix for LSSs \cite{Petreczky2015,sarkar2019nonparametric}.
Thus we define the Hankel-matrix
$H_{\alpha, \beta} \in \mathbb{R}^{n \times n}$
%elementwise, namely
as
\begin{align*}
    (H_{\alpha, \beta})_{i, j} = M_{q_j \mu_j \eta_i}[l_j].
\end{align*}
%and they determine an identifiable LSS parametrization \cite{petreczky2010spaces}.
\begin{lemma}[\cite{MertBastug:TAC}]
\label{lemma:hk:1}
%\textcolor{black}{\textbf{Under what assumptions?}} 
There exists an $n$-row selection $\alpha$ and $n$-column selection $\beta$,
such that 
    $\text{rank}(H_{\alpha, \beta}) = n$.
%$Q_\subseteq Q^*$
%and  $Q_\beta =  Q^* \times Q$ as defined above, such that
%\begin{itemize}
%    \item $\forall \alpha_i \in \alpha: |\alpha_i| \leq n - 1$,
%    \item $\forall \beta^0_i q_i \in \beta: |\beta^0_i| \leq n - 1$,
%    \item 
%\end{itemize}
\end{lemma}
%In order to formulate the reduced basis 
%Ho-Kalman algorithm, 
\textcolor{black}{We introduce} the following matrices 
$H_{\alpha, q, \beta} \in \mathbb{R}^{n \times n}$,
 $H_{\alpha, q} \in \mathbb{R}^{n \times m}$,
$H_{\beta} \in \mathbb{R}^{1 \times n}$, where 
\begin{align*}
%&H_{\alpha, q, \beta} \in \mathbb{R}^{n \times n} \text{ s.t. } 
& (H_{\alpha, q, \beta})_{i, j} = M_{q_j\mu_j q \eta_i}[l_j], ~ 
%H_{\alpha, q} \in \mathbb{R}^{n \times 1} \text{ s.t. } 
(H_{\alpha, q})_{i,j} = M_{\eta_i q }[l_j], \\
%& H_{\beta} \in \mathbb{R}^{1 \times n} \text{ s.t. } 
& (H_\beta)_{1, j} = M_{q_j \mu_j}[l_j].
\end{align*}
The reduced basis Ho-Kalman algorithm computes $(\{\bar{A}_q,\bar{B}_q\}_{q \in Q}, \bar{C},\bar{D})$
defined as 
%Given the true Markov parameters, we can obtain the system by the following equalities for $q \in Q$ \textcolor{black}{\textbf{TODO: cite ho kalman}}.
$\bar{A}_q = H_{\alpha, \beta}^{-1} H_{\alpha, q, \beta}$\\
$\bar{B}_q= H_{\alpha, \beta}^{-1} H_{\alpha, q} ~  
\bar{C} = H_{\beta}$ and 
$\bar{D} = M_{\epsilon}$.
%y_{t}u_{t}}
%\td{Q: no svd factorization? why would this work?}

\begin{lemma} [\cite{petreczky2010spaces}]
\label{lemma:hk:2}
If $\text{rank}( H_{\alpha, \beta}) = n$, there exists $T \in \mathbb{R}^{n \times n}$ such that $\text{det}(T) \neq 0$ and for all $q \in Q$ we have
$\bar{A}_q = TA_q T^{-1}, \quad \bar{B}_q = TB_q, \quad \bar{C}= C_ T^{-1}$.
\end{lemma}
%%
%%\begin{color}{black}
%%From Lemma \ref{lemma:hk:2} it follows that
%%for $\mathcal{W} = \{vqs | v \in \alpha, q \in Q, s \in \beta\} \cup \{sv | s \in \alpha, v \in \beta\}
%%\cup \beta \cup  \{sq |  s \in \alpha, q \in Q\}$,  
%%the Markov-parameters $\{M_w\}_{w \in \mathcal{W}}$ determine all the Hankel matrices described above, hence
%%they determine the system matrices up to an isomorphism.
%%That is, in order to estimate the system matrices, it is sufficient to estimate the Markov parameters 
%%$\{M_w\}_{w \in \mathcal{W}}$ for a set $\mathcal{W}$
%%arising from nice selections $\alpha,\beta$ satisfying the condition of Lemma \ref{lemma:hk:2}. In addition,
%%Lemma \ref{lemma:hk:1} ensures that such nice selections always exist. Note that $\abs{\mathcal{W}} \leq n^2n_Q + n^2 + n + nn_Q \leq 2n^2(n_Q + 1)$, i.e., for the reduced basis Ho-Kalman algorithm above, the knowledge of $O(n^2)$ Markov parameters is sufficient. This is in contrast to naive version of Ho-Kalman algorithm \cite{Petreczky2015,sarkar2019nonparametric}, which requires the knowledge of $O(n_Q^n)$ Markov parameters.
%The price to pay is that suitable $n$-selections have to be found. As there  are $O(n_Q^{n})$ nice selections, finding  selections which satisfy Lemma \ref{lemma:hk:2} takes $O(n_Q^n)$ steps. That is, by using the above Ho-Kalman algorithm, in the worst case  we traded storage complexity for time complexity. However,for learning, storage complexity is often more critical then time complexity.
%\end{color}

By Lemma \ref{lemma:hk:2}, the \textcolor{black}{finite collection} 
\(\{M_w\}_{w \in \mathcal{W}}\), with  
%\begin{align*}& 
\begin{equation}
\label{setw}
\begin{split}
& \mathcal{W} = \{\nu_i q \mu_j q_j, \nu_i\mu_jq_j, \nu_iq, \mu_jq_j \}_{i,j=1}^{n},  \\
& \abs{\mathcal{W}} \leq \textcolor{black}{ n^2(n_Q + 1)+n_Qnm+n}
\end{split}
\end{equation}
of Markov parameters 
%| v \in \alpha, q,\sigma \in Q, \exists l %\in \mathbb{N}: (s,\sigma,l) \in \beta\} \cup \\
%& \{sv\sigma | s \in \alpha, \sigma \in Q, \exists l \in (v,\sigma,l) \in \beta\} \\
%& \cup \beta \cup \{sq \mid s \in \alpha, q \in Q\},
%\end{align*}
determine all the Hankel matrices \textcolor{black}{above},
%and thus the system matrices 
%used in the algorithm above, 
thus estimating the system matrices reduces to estimating \(\{M_w\}_{w \in \mathcal{W}}\).
\textcolor{black}{By \eqref{setw} 
the realization algorithm above requires $O(n^2)$}
%for a suitable \(\mathcal{W}\) derived from selections \(\alpha, \beta\) satisfying Lemma \ref{lemma:hk:2}, which always exist by Lemma \ref{lemma:hk:1}. }
%The size of \(\mathcal{W}\) satisfies \( \), meaning that he reduced-basis Ho-Kalman algorithm requires only \(O(n^2)\) 
Markov parameters, in contrast to the naive version \cite{Petreczky2015,sarkar2019nonparametric}, which requires \(O(n_Q^n)\). The trade-off is the need to find suitable \(n\)-selections, which takes \(O(n_Q^n)\) steps, as there are \(O(n_Q^n)\) choices. However, in \textcolor{black}{learning}, storage complexity is often more critical, than time complexity.
%%\textbf{System identification algorithm}
%% \textcolor{black}{
%% Note that the Markov parameters can be computed using the expectations \eqref{eq:Mw}. 
%% Suppose that we can measure a finite portion
%% $\mathcal{S}_N=\{(y(t),u(t),q(t))\}_{t=0}^{N}$ of a sample path
%% of $(\by(t),\bu(t),\bq(t))$.  
%%  Then, if we replace the expectations in \eqref{eq:Mw}
%%  by their empirical counterparts constructed from
%%  $\mathcal{S}_N$, and use the thus arising estimates of
%% estimates of Markov parameters instead of the true ones
%% in the Ho-Kalman algorithm above, then
%% we obtain the system identification algorithm described in \cite{CoxLPVSS,CoxTothSubspace, PETRECZKY2023101308, Rouphael2024},  which we refer to as empirical Ho-Kalman algorithm.}
%%
%% \textcolor{black}{Since in this paper we are interested in statistical properties of this algorithm, we would like to view it as a statisics, i.e., we would like to present the estimated Markov parameters and the system matrices as random variables which are functions of the random
%% variables $\{\by(t),\bu(t),\bq(t))\}_{t=0}^{N}$. }
%% To this end, let $\alpha, \beta$ be such that $\text{rank}(H_{\alpha, \beta}) = n$ \cite{CoxLPVSS}, and
%% we define the following random variables:\\
%\begin{itemize}
The Markov parameters can be computed from the expectations in \eqref{eq:Mw}. \textcolor{black}{Given 
$\mathcal{S}_N = \{(y(t),u(t),q(t))\}_{t=0}^{N}$ a finite-sample path}, we replace these expectations with their empirical counterparts. Substituting these estimates in the Ho-Kalman algorithm yields the empirical Ho-Kalman algorithm, as described in \cite{CoxLPVSS,CoxTothSubspace, PETRECZKY2023101308, Rouphael2024}.  

Since we focus on the statistical properties of this algorithm,
%we treat the algorithm as a statistic of the data, interpreting
\textcolor{black}{we interpret}
the estimated Markov parameters and system matrices \textcolor{black}{as the following} random variables. 
%depending on \(\{\by(t),\bu(t),\bq(t)\}_{t=0}^{N}\)
%To formalize this, let \(\alpha, \beta\) be such that \(\text{rank}(H_{\alpha, \beta}) = n\) \cite{CoxLPVSS}, and define the following random variables:  

\noindent \textbf{Empirical Markov parameters} are defined as
%Compute the estimate 
\begin{align}
\label{eq:mhat}
\!\! \hMw = \frac{1}{N-|w|} \sum_{t=|w|+1}^N \by(t) \bz_{w}^\bu(t)\Sigma_u^{-1}p_w^{-1}. \!\!
\end{align}
\textbf{Empirical Hankel matrices $\bHhat_{\alpha, \beta}, \bHhat_{\alpha, q, \beta}, \bHhat_{\alpha, q}, \bHhat_\beta$ } are constructed in the same way as $H_{\alpha, \beta}, H_{\alpha, q, \beta},  H_{\alpha, q}, H_\beta$ but by using $\hMw$ instead of $M_w$. \\
\textbf{Estimates of the system matrices} are obtained as
%Apply the classical Ho-Kalman to $\bHhat_{\alpha, \beta}, \bHhat_{\alpha, q, \beta}, \bHhat_{\alpha, q}, \bHhat_\beta$:
\begin{align*}
\bAhat_q = \bHhat_{\alpha, \beta}^{-1} \bHhat_{\alpha, q, \beta}, \quad \bBhat_q = \bHhat_{\alpha, \beta}^{-1}\bHhat_{\alpha, q}, \quad \bChat = \bHhat_{\beta}.
\end{align*}
%\end{description}
%\textcolor{black}{That is, if
%the training set $\mathcal{S}_N$ corresponds to the random element 
%$\omega \in \Omega$, i.e.,
%$y(t)=\by(t)(\omega)$, $u(t)=\bu(t)(\omega)$,
%  $q(t)=\bq(t)(\omega)$, $t=0,\ldots,N$, then
%$\hMw(\omega)$
%and 
%$\{\bAhat_q(\omega),\bBhat_q(\omega)\}_{q %\in Q},\bChat(\omega)\}$
%represent the empirical estimates of the Markov-parameter $M_w$ and the outcome of the empirical Ho-Kalman identification algorithm respectively, when applied to the data
%$\mathcal{S}_N$ }
  If the training set \(\mathcal{S}_N\) corresponds to the random element \(\omega \in \Omega\), i.e.,  
\( y(t) = \by(t)(\omega), u(t) = \bu(t)(\omega), q(t) = \bq(t)(\omega) \) for \( t = 0, \dots, N \),  
then \(\hMw(\omega)\) and  
\(\{\bAhat_q(\omega), \bBhat_q(\omega)\}_{q \in Q}, \bChat(\omega)\}\)  
represent the empirical estimates of the Markov parameters \(M_w\), and the system matrices obtained from the empirical Ho-Kalman algorithm applied to \(\mathcal{S}_N\), respectively.

\section{Main result}
\label{sec:main}

%\textcolor{blue}{The main result of this paper is a finite-sample bound on the distance in Frobenius norm, between the output of the classical and the empirical Ho-Kalman algorithm. First, we bound the absolute difference between the true and empirical Markov parameters in the following Lemma.}
The main result of this paper is a finite-sample bound on the Frobenius norm distance between the outputs of the classical and empirical Ho-Kalman algorithms. We first bound the absolute difference between the true and empirical Markov parameters in the following Lemma.
\begin{lemma}
\label{lemma:bound_M}
    %\textcolor{black}{\textbf{Under TODO assumptions}}

       Under Assumption \ref{ass:main},
    if $N > 2(2n+1)$
    %$N$ is the length of
    %the trajectory
    \begin{align}
         \mathbb{P}\left[\|\Mw - \hMw\|_F \leq K(\delta, N)\right] > 1 - \delta
    \end{align}
    \begin{align}
        &  \widetilde{K}(\delta,N)=K_0\sqrt{8 \textcolor{black}{\log\left(\frac{2 m|\mathcal{W}|}{\delta}\right) / N}},
   \nonumber \\       
        & \textcolor{black}{K_{-1}= 
     \|\Sigma^{-1}_u\|_F / \min\limits_{w \in \mathcal{W}} p_w}, \quad K(\delta, N) = 
          \textcolor{black}{ \textcolor{black}{\sqrt{m}}K_{-1} 
          \widetilde{K}(\delta,N)} 
        %\nonumber \\
     %&    , \quad  
     \nonumber \\
     %\sqrt{m} 
        & 
         K_0 = \frac{\textcolor{black}{\sqrt{m'}K'_uK_M}}{1-\gamma} + K_\theta,  \nonumber \quad \textcolor{black}{K_u' = \max\{K_u, K_u^2\}} \\
	    & K_\theta =  2(\textcolor{black}{m'}K_u + n_Q)\textcolor{black}{\sqrt{m'}K_u'K_M K_\gamma} \nonumber \\
       & K_\gamma=\frac{\gamma  + (2n + 1)\textcolor{black}{(n+2)}(1-\gamma)^2\textcolor{black}{+(1-\gamma)^3}}{(1-\gamma)^3}
        \nonumber
        %2(n_Q + 1)n^2m
    \end{align}
    %where 
    %\textcolor{black}{$p_{min} = \min\limits_{w \in \mathcal{W}} p_w$}
    %\ge (p^*)^{2n+1}$,
    %and $\mathbb{P}[\mathbf{q}(t)=q^*]=p^*$},
     %$$,
    % \\ %\textcolor{blue}{$K_1(n,m)= (n_Q + 1)n^2+n_Qnm+n, 
    %.
%    \begin{align}
%        \mathbb{P}\left[\forall w \in \mathcal{W}:  \abs{\Mw - \hMw} < K(\delta, N) \right] > 1 - \delta
%    \end{align}
    %\textcolor{blue}{
    %$K(\delta, N) = \frac{\sqrt{m}}{\norm{\Sigma^{-1}_u}_F} \sqrt{\frac{K_1}{N-\sup_{w \in \mathcal{W} }|w|}  \log\left(\frac{2(n_Q + 1)n^2m}{\delta}\right)}$ 
    %and $N$ denotes the size of the random sample.
\end{lemma}

    \textcolor{black}{Note that
    $|\mathcal{W}|$ can be upper bounded by 
    \eqref{setw}. Additionally, if $p_{*}=\min\limits_{q \in Q} p_q$
    then $1/\min\limits_{w \in \mathcal{W}} p_w \le 1/p_{*}^{2n+1}$.
    }
    
\begin{proof}
Let us fix $\lambda > 0$, an index $l \in [1,m]$, and $w = w_1\ldots w_k$. Let 
$\htMw[l]= \frac{1}{N-|w|}\sum\limits_{t=|w|+1}^N \by(t)\bz_w^{\bu[l]}(t)$ 
and 
$\tMw[l] = \mathbb{E}[\by(t)\bz_w^{\bu[l]}(t)]$.
%and define 
%$\widetilde{K}(\delta,N)=\textcolor{black}{K_0 \sqrt{8\log\left(\frac{2m|\mathcal{W}|}{\delta}\right)/N}}$.
\textcolor{black}{It is enough to show that}
\begin{align}
\label{eq:wl}
    \mathbb{P}\big[|\tMw[l] - \htMw[l]| > \widetilde{K}(\delta, N)\big] \leq
    \delta m ^{-1}|\mathcal{W}|^{-1}.
\end{align}
The reason is that since 
    \( \|\Mw - \hMw\|^2_F 
    %\|(\tMw - \htMw)\frac{\Sigma^{-1}_u}{p_w}\|^2_F 
    \leq  
     \|\tMw - \htMw\|^2_F \frac{\|\Sigma^{-1}_u\|^2_F}{p^2_w} \)
and $\|\tMw - \htMw\|^2_F 
    = \sum\limits_{l=1}^m|\tMw[l] - \htMw[l]|^2$, 
taking the union-bound \cite[Lemma 2.2]{shalev2014understanding} over the set of indices $l$, and the sum, give
   $\mathbb{P}[ \|\tMw - \htMw\|_F > \sqrt{m}\widetilde{K}(\delta, N)]$\\
   $\textcolor{black}{\leq} \delta |\mathcal{W}|^{-1}$, then the union-bound over $\mathcal{W}$ and the complement rule yields the result of Lemma \ref{lemma:bound_M}.
 
	Thus, it is enough to show $\eqref{eq:wl}$. % for a fixed $w$ and $l$. 
	%for any  $\epsilon \in \{1,-1\}$. 
	\textcolor{black}{To this end, for} 
%	\begin{equation}
%	\label{eq:wl1}
%		\mathbb{P}\big[ \epsilon (\tMw[l] - \htMw[l]) > \epsilon \widetilde{K}(\delta, N)\big] \leq
%    \delta/2 m ^{-1}|\mathcal{W}|^{-1}
%	\end{equation}
%	where $\epsilon \in \{1,-1\}$, from which \eqref{eq:wl} follows by the union bound. }
%For any 
	\textcolor{black}{$\lambda > 0$, $\vartheta \in \{-1,1\}$,}
the Chernoff-bound \textcolor{black}{\cite[Lemma B.3]{shalev2014understanding}}
%the Markov-inequality 
leads to
   \begin{align*}
	   &\mathbb{P}\big[ \textcolor{black}{\vartheta (\tMw[l] - \htMw[l])} >  K\big] %\leq 
 %   \mathbb{P}\left[ \e^{\lambda \left|\tMw[l] - \htMw[l] \right|} > \e^{\lambda K}\right] \\
	   \leq \frac{\mathbb{E}[ \e^{\lambda \textcolor{black}{\vartheta} \textcolor{black}{(\tMw[l] - \htMw[l])}}]}{\e^{\lambda K}} = \Theta(\lambda)
\end{align*}
Next, we wish to apply \cite[Theorem 6.6]{alquier2012} to the numerator. \begin{theorem}{\cite[Theorem 6.6]{alquier2012}}
\label{thm:6.6}
    Let $\mathcal{X}$ be a Banach space with the norm $\norm{\cdot}$ and let $\mathbf{Y} = \{\mathbf{Y}_t\}_{t \in \mathbb{Z}}$ be a bounded, stationary stochastic process 
    \textcolor{black}{such that $\norm{\mathbf{Y}_0} < C_{\mathbf{Y}}$}.
    %ith a bounded distribution $\pi_0$ on $\mathcal{X}^{\mathbb{Z}}$.
    Let $h:\mathcal{X}^{\mathbb{Z}} \rightarrow \mathbb{R}$ be a 1-Lipschitz function, \textcolor{black}{i.e.,}  
    $\forall \{x_t\}, \{y_t\} \in \mathcal{X}^N$ we have
        $\left| h(x_1,\ldots, x_N) - h(y_1,\ldots,y_N) \right| \leq \sum\limits_{i=1}^N \norm{x_i - y_i}$.
    Then, for any $\nu > 0$ we have
    \begin{align*}
        &\mathbb{E} \left[\e^{\nu \left( \mathbb{E}\left[h(\mathbf{Y}_1,\ldots,\mathbf{Y}_N)\right]
        - h(\mathbf{Y}_1,\ldots, \mathbf{Y}_N) \right)} \right] 
        \leq \e^{\frac{\nu^2 N \left( \textcolor{black}{C_{\mathbf{Y}}} + \theta \right)^2}{2}},
    \end{align*}
    where $\theta$ is defined in \cite[Chapter 3.1]{alquier2012} as $\theta_{\infty, N}(1)$. 
    %\textcolor{black}{$\norm{\mathbf{Y}_0}_\infty$ 
    %denotes the essential supremum of the random variable 
    %$\mathbf{Y}_0$}
\end{theorem}
%The exact definition of 
%\textcolor{black}{We will bound $\theta = \theta_{\infty, N}(1)$ later.}
%is not strictly relevant for this result. We remark that it is not necessarily finite, however, we will show that under the quadratic stability assumption it is bounded, as it is shown in Lemma \ref{lemma:ktheta}.
Let $h(x_1, \ldots, x_N) = \sum\limits_{t=|w|+1}^N x_t$,
$\bX_i = \textcolor{black}{\vartheta} \by(i) \bz_{w}^{\bu[l]}(i)$ and  $\mathcal{X}=\mathbb{R}$.
%In light of Theorem \ref{thm:6.6},
%\begin{align}
%    X_i &= y_iz_{w}^u(i) \label{eq:x_i} \\
%    h(x_1, \ldots, x_N) &= \sum\limits_{t = |w| + 1} ^N x_t. \label{eq:smallh}
%\end{align}
\textcolor{black}{It is clear that $h$ is 1-Lipschitz} and $ \textcolor{black}{\vartheta} \htMw[l] = (N - |w|)^{-1} h(\bX_1, \ldots, \bX_N)$,
and \textcolor{black}{$\mathbb{E}[ \htMw[l] ] = \tMw[l]$}. % for all $w \in \mathcal{W}$
By Lemma \ref{lem:stab:stat},
$\{ \by(t), \bu(t), \bq(t) \}$ is jointly stationary and bounded, and hence
$\{\mathbf{X}_i\}_{i \in \mathbb {Z}}$ is a stationary bounded process
with $\textcolor{black}{\norm{\mathbf{X}_0}_2 \le K_yK_u \le \frac{\sqrt{m'} K'_uK_M}{1-\gamma}}$.
Moreover, the coefficient $\theta$ from Theorem \ref{thm:6.6}
is finite for $\{\mathbf{X}_i\}_{i \in \mathbb {Z}}$.
\begin{lemma}
\label{lemma:ktheta}
For $K_\theta$ from Lemma \ref{lemma:bound_M} we have $\theta \leq K_{\theta}$.
\end{lemma}
The proof is in Section \ref{sec:proofs}.
%Moreover, by Assumption \ref{ass:main}
%$\{\bu(t)\}$ and $\{\bq(t)\}$ are i.i.d respectively, and are independent. %\textbf{O $S$ is bounded.
Therefore %$\textcolor{black}{\pm} 
$\{\bX_i \}$ satisfies the conditions of Theorem \ref{thm:6.6}, hence we can apply it with $\nu = \lambda(N - |w|)^{-1}$ 
to $\textcolor{black}{\{\bX_i \}}$
with $K_u K_y + \theta \leq K_0$, yielding
\begin{align*}
   %\frac{
   \mathbb{E}[ \e^{\textcolor{black}{\vartheta \lambda  (\tMw[l] - \htMw[l] })}]
   %}{\e^{\lambda K}}
    \leq 
    %\frac{
    %\e^{\frac{\lambda^2 N}{2(N - |w|)^2} \left(\textcolor{black}{\frac{\sqrt{m'}K'_uK_M}{1-\gamma}+K_{\theta}} \right)^2}
    \textcolor{black}{\e^{\frac{\lambda^2 N}{2(N - |w|)^2} K_0^2}}
    %}{\e^{\lambda K}}
\end{align*}
%We show $\theta < \infty$ under the quadratic stability assumption.
%For $K_0$ defined in Lemma \ref{lemma:bound_M} (noting that , 
%\textcolor{black}{it holds that}
%\begin{align*}
   %\frac{\mathbb{E}\left[ \e^{\lambda | \tMw[l] - \htMw[l] |} \right]}{\e^{\lambda K}}
   \textcolor{black}{Hence, \(   \Theta(\lambda)
    \leq \e^{\frac{\lambda^2 N}{2(N - |w|)^2} K_0^2 -\lambda K} =: e^{F(\lambda)}. \)}
%\end{align*}
%By minimizing $F(\lambda)$ using standard techniques
It follows the minimal value 
$F(\lambda_{\star})$ of $F$ is
$-\frac{K^2 (N-|w|)}{2 NK_0^2}$ \textcolor{black}{and it is achieved at} \textcolor{black}{ $\lambda_{\star}=(K K_0^2 (N-|w|)^2)/N$,}
%Due to the monotonicity of the exponential function, we can take the minimum of $e^{F(\lambda)}$ by setting the derivative $F'(\lambda)$ to zero,
%obtaining
%\begin{align*}
%    &F'(\lambda) = \frac{\lambda N K_0^2}{(N-|w|)^2} - K = 0%
%\end{align*}
%Solving the equation above we obtain
%\begin{align*}
%\(    \lambda_{\min} = \frac{K(N-|w|)^2}{N K_0^2} \)
%as minimum point of $F(\lambda)$. In particular,
%$F(\lambda_{\min})=-\frac{K^2 (N-|w|)}{2 NK_0^2}$.
%\end{align*}
%Substituting back we obtain
%\begin{align*}
   %\frac{\mathbb{E}\left[ \e^{\lambda | \tMw[l] - \htMw[l] |} \right]}{\e^{\lambda K}}
\textcolor{black}{implying}  \textcolor{black}{\( \Theta(\lambda_{\star}) \leq F(\lambda_{\star})\).}
	%e^{\frac{-K^2(N-|w|)^2}{2K_0^2 N}} \).}
By 
using
$N > 2(2n+1) \ge 2|w|$ %and hence $N-|w| \ge 0.5N$.
 and choosing $K=\widetilde{K}(\delta,N)$ 
it follows that 
%it follows that 
%and by solving
\textcolor{black}{\( %e^{\frac{-K^2(N-|w|)^2}{2K_0^2 N}} 
	F(\lambda_{\star}) \le 0.5 \delta\textcolor{black}{|\mathcal{W}|^{-1}m^{-1}}\).}
    
\textcolor{black}{As a result,  $\mathbb{P}\big[ \textcolor{black}{\vartheta (\tMw[l] - \htMw[l])} >  \widetilde{K}(\delta,N)\big] \le \Theta(\lambda_{*})$} \\  \textcolor{black}{ $\le 0.5 \delta\textcolor{black}{|\mathcal{W}|^{-1}m^{-1}}$.
	Since
	$|\tMw[l] - \htMw[l]| <  K$ is equivalent to 
	$\forall \vartheta \in \{1,-1\}: \vartheta (\tMw[l] - \htMw[l]) <  K$, hence
	\eqref{eq:wl} follows using the union bound. 
}
%\textcolor{black}{2((n_Q + 1)n^2+n m n_Q+n)}}\).}
%in $K$ it follows that
%After rearrangement, we have \\
    %$K = \sqrt{\frac{2K_0N}{(N-|%w|)^2}\log\left(\frac{2(n_Q + 1)n^2 m}%{\delta}\right)}$.}
%Finally, $|\mathcal{W}| \leq (n_Q + 1)n^2\textcolor{black}{+n_Q n m +n}$,
%\textcolor{black}{The application of the union bound} %$p_w \leq 1$, 
%$K \leq \textcolor{black}{\tilde{K}}(\delta, N)$ 
%and $N > 2(2n+1)$
%conclude the proof
%of \eqref{eq:wl}.
\end{proof} 

%It grows with the bound on the input and decreases with the Lyapunov exponent $\gamma$ of the underlying system. It also grows with the constant $K_{\theta}$, which captures the mixing properties of the output generated by the system. As small $K_{\theta}$ is, the closer the output to being i.i.d. In turn, $K_{\theta}$ is increasing in the bound of the input and the noises, and in the Lyapunov exponent $\gamma$.
%That is, the smaller $\gamma$ is, i.e., the more stable the underlying system is, the smaller is the constant $K_0$. Finally, $K_0$ is $O(n^2)$ and it is linear in $n_Q$. 
%
The next lemma establishes the sensitivity of the Hankel-matrices in the Ho-Kalman algorithm to perturbations of Markov-parameters. 
%This enables us to provide bounds on the differences between the 
%Hankel-matrices for the true and the estimated Markov-parameters. 
%
\begin{lemma}
\label{lemma:bound_H}
    %Let $\mathcal{H}$ be a
    %function which maps any
    \textcolor{black}{Let $n_* = \max\{n,m\}$ and $K_{-1}$ from Lemma \ref{lemma:bound_M}.} For any 
    collection of $1 \times m$ matrices
    $\{\widehat{M}_w\}_{w \in \mathcal{W}}$,
    and any pair of matrices
    $(H,\widehat{H})=(H_{\kappa},\widehat{H}_{\kappa})$, 
    \textcolor{black}{on condition that}  $\kappa \in \mathcal{K}=\{(\alpha,\beta)\} \cup \{\beta\} \cup \{(\alpha,q,\beta), (\alpha,q)\}_{q \in Q}$, where
    %$\{(H_{\alpha,\beta},\hat{H}_{\alpha,\beta}), H_{\alpha,q,\beta},\hat{H}_{\alpha,q,\beta}), (H_{\alpha,q},\hat{H}_{\alpha,q}),(H_{\beta},\hat{H}_{\beta})\mid q \in Q\}$, 
    $\widehat{H}_{\kappa}$
    is defined as $H_{\kappa}$
    %$\hat{H}_{\alpha,\beta}$, $\hat{H}_{\alpha,q,\beta}$, $\hat{H}_{\beta}, \hat{H}_{\alpha,q}$
    %the matrices which 
    %are defined
    %as $H_{\alpha,\beta}$, $H_{\alpha,q,\beta}$, $H_{\beta}$, $H_{\alpha,q}$ 
    but with $M_w$
    replaced by 
    $\widehat{M}_w$ 
    :
    %the following holds. 
    %Consider a set of words $\mathcal{W}$ and the functions $%\mathcal{H}, \widehat{\mathcal{H}}:
    %\mathbb{R}^{|W| \times n \times n} \rightarrow \mathbb{R}^{ n \times n}$ such that 
    %$\mathcal{H}(M) = H$ and
    %$\widehat{\mathcal{H}}(\widehat{M}) = 
    %\widehat{H}$. The functions $\mathcal{H}$ and $\widehat{\mathcal{H}}$ represent the algorithms that produce an arbitrary Hankel matrix $H$ and its approximation $\widehat{H}$ from the true and empirical Markov parameters, respectively. 
    %For $\varsigma = \sigma_n(H)$. 
    %the following holds:
    %\begin{enumerate}
    %    \item 
    \begin{align}
       & 
       \textcolor{black}{\|H\|_F \leq 
        n_* K_{-1} K_yK_u}
        %\frac{\|\Sigma^{-1}_u\|_F}{ \min\limits_{w \in \mathcal{W}} p_w} K_yK_u}
       %\sqrt{\sum_{w \in \mathcal{W}} \norm{M_w}_F^2 } 
       \mbox{ and }
        \|\widehat{H}^{-1}\|_F \leq \frac{2}{\sigma_n(H)}, \quad  
        \label{eq:boundH:1}
        \\
        %\item  \textcolor{blue}{$\
       &  \|\widehat{H} - H\|_F \leq \sqrt{\sum\limits_{w \in \mathcal{W}} \|M_w - \widehat{M}_w\|_F^2} \label{eq:boundH:2}
    \end{align}
        %If for all $w \in %\mathcal{W}$ we have
        If $\forall w \in \mathcal{W}: \|\Mw - \widehat{M}_w\|_F < \frac{\sigma_n(H)}{2\textcolor{black}{|\mathcal{W}|} \sqrt{n}}$, then
    \begin{align}        
    \|\widehat{H}^{-1} - H^{-1}\|_F \leq \frac{4\sqrt{2}}{\sigma_n^2(H)} \sqrt{\sum\limits_{w \in \mathcal{W}} \|M_w - \widehat{M}_w\|^2_F} \label{eq:boundH:3}
    \end{align}
   % \end{enumerate} 
\end{lemma}

The proof is in Section \ref{sec:proofs}. 
By using the fact that the matrices returned by the Ho-Kalman algorithms are obtained by arithmetic operations from the Hankel matrices, Lemma \ref{lemma:bound_M} and Lemma \ref{lemma:bound_H}, we obtain the following Theorem.
%are ready to state the main Theorem where we bound the estimation error of the above described algorithm.}

\begin{theorem}[Main result]
\label{thm:main}
Under Assumption \ref{ass:main}, 
if $K(\delta, N)$, \textcolor{black}{$K_{-1}$} as in Lemma \ref{lemma:bound_M}, \textcolor{black}{$n_*$ as in Lemma \ref{lemma:bound_H}},
\textcolor{black}{$K_y,K_u$ from Lemma \ref{lem:stab:stat} and Assumption \ref{ass:main}},
%$\rank(H_{\alpha,\beta)=n$ 
%$\varsigma =\sigma_n(H_{\alpha,\beta})$,
and $N$ large enough for $N > 2(2n+1)$ and $K(\delta, N) \leq \frac{\sigma_n(H_{\alpha,\beta})}{2\textcolor{black}{|\mathcal{W}|}\sqrt{n}}$, then
%\min\limits_{w \in \mathcal{W}} $, 
% Let $K^{\mathcal{W}}(\delta, N) =
%\sqrt{\sum\limits_{w \in \mathcal{W}}K^w(\delta, N)^2}.$
\begin{align}
    %&\mathbb{P}\left[
    %\max\limits_{(\bZhat, \bar{Z}) \in S} \norm{\bZhat - \bar{Z}}_F < K_2 K(\delta, N) \right] > 1 - \delta. 
    \nonumber
    \mathbf{EstErr} = &\max\limits_{q \in Q}\{\|\bAhat_q - \bar{A}_q\|_F, 
    \|\bBhat_q - \bar{B}_q\|_F , 
    \|\bChat-\bar{C}\|_F \}\\  
    & \mathbb{P}\Big[ \mathbf{EstErr} < K_2 K(\delta, N) \Big] > 1-\delta
    \label{eq:main1} 
\end{align}
\[ \!\! K_2 =  
%\sqrt{2(n_Q + 1)}n
\textcolor{black}{\sqrt{|\mathcal{W}|}} \max\Big\{1, \frac{2\sigma_n(H_{\alpha,\beta})+
\textcolor{black}{n_*K_{-1}K_yK_u
%\max_{q \in Q}\{\|H_{\alpha,q \beta}\|_F,\|H_{\alpha,q}\|_F}\}
4\sqrt{2}}}{\sigma_n^2(H_{\alpha,\beta})}\Big\} \!\!
\]
%where \textcolor{blue}{$S = \{(\bAhat_q, \bar{A}_q)\}_{q \in Q} \cup \{(\bBhat_q, \bar{B}_q)\}_{q \in Q} \cup (\bChat, \bar{C})$, 
%\textcolor{blue}{where 
%$.$
%}

\end{theorem}

\begin{proof}
    The existence of $N$ is guaranteed by Lemma \ref{lemma:bound_M} and that $K(\delta, N)$ is strictly decreasing in $N$.

    First we prove \eqref{eq:main1} for $A_q$. For a fixed $\alpha, \beta$ and $q \in Q$ let $H_1 = H_{\alpha, \beta}$, $H_2 = H_{\alpha, q, \beta}$, $\bHhat_1$ = $\bHhat_{\alpha, \beta}$ and $\bHhat_2 = \bHhat_{\alpha, q, \beta}$.
    Due to the definitions and Lemma \ref{lemma:bound_H}, for $\varsigma = \sigma_n(H_1)$ 
    \begin{align*}
        &\|\bAhat_q - \bar{A}_q\|_F = \|\bHhat_1^{-1} \bHhat_2  - H_1^{-1} H_2\|_F \\
         \textcolor{black}{ \leq} &\|\bHhat_1^{-1}\|_F\|\bHhat_2 - H_2\|_F + \|H_2\|_F \|\bHhat_1^{-1} - H_1^{-1}\|_F \\
         %\textcolor{black}{ \leq} &\textcolor{black}{(2/\varsigma + n^2K_uK_y 4\sqrt{2}/\varsigma^2 )}\sqrt{\sum_{w \in \mathcal{W}}\|\Mw - \hMw\|_F^2}
         \textcolor{black}{ \leq} &\textcolor{black}{K_2}\sqrt{\sum_{w \in \mathcal{W}}\|\Mw - \hMw\|_F^2}
    \end{align*}
    Therefore by Lemma \ref{lemma:bound_M} and basic properties of probability,
    \begin{align*}
        &\mathbb{P}[\|\bAhat_q - \bar{A}_q\|_F < K_2 K(\delta, N) ]  \\
       \textcolor{black}{\geq} &\mathbb{P}[\|\Mw - \hMw\|_F < K(\delta, N) ] > 1 - \delta
    \end{align*}
    For $B_q$ it is the same proof with $H_2 = H_{\alpha, q}$.
    For $C$ we have $\|\bChat - \bar{C}\|^2_F = \|\bHhat_\beta - H_\beta\|^2_F \leq \sum\limits_{w \in \mathcal{W}}\|\Mw - \hMw\|_F^2$ which is due to Lemma \ref{lemma:bound_H}.
\end{proof}

\textbf{Discussion.}
The term $K_2 K(\delta,N)$ is $O\left(\sqrt{\log(1/\delta)/N}\right)$, since $K_2$ is independent of $N$, matching standard PAC bounds \cite{alquier2012,shalev2014understanding}. Thus, as $N \to \infty$, the estimation error vanishes with probability $1 - \delta$, ensuring statistical consistency. The constant $K_0$ in $K(\delta,N)$ is
%$O\left(\frac{1}{(1-\gamma)^3}\right)$,
$O((1-\gamma)^{-3})$,
due to 
$K_{\theta}$ being 
$O((1-\gamma)^{-3})$,
%$O\left(\frac{1}{(1-\gamma)^3}\right)$,
therefore higher degree of  stability (smaller $\gamma$) 
yields smaller estimation error of Markov parameters. The term $K_{\theta}$ upper bounds a mixing coefficient. % and also decreases as stability grows. 
The bound is inversely proportional to the input covariance, which is consistent with persistence of excitation.
%namely higher input variance leads to tighter bounds.
\textcolor{black}{Moreover, as $K(\delta,N)=O\bigl(p_*^{-(2n+1)}N^{-1/2}\bigr)$, for fixed $N$, small $p_*$ both inflates the bound and means too few samples visit the mode $q_*$ (as $\mathbb{P}(\bq(t)=q_*)=p_*$), so no algorithm can learn the system reliably. Allowing $N$ to grow arbitrarily large compensates for any small $p_*$, ensuring Ho–Kalman succeeds once enough visits to $q_*$ are collected.
}
Finally, 
%as $K_2$ is $O(n^3)$,
%hence $K_0$ is $O(n)$,
%so the overall bound scales polynomially with the number of states, as expected given the number of parameters to estimate.
$K_2$ and $K_0$ \textcolor{black}{increase polynomially} in the number of states,
discrete modes, and bounds on the input noise.
The condition $K(\delta,N) <\frac{\sigma_n(H_{\alpha,\beta})}{\sqrt{n}\textcolor{black}{|\mathcal{W}|}}$ ensures the estimated Hankel matrix $\bHhat_{\alpha,\beta}$ is invertible with high probability, enabling the Ho-Kalman algorithm, and it holds if
%$N = \Omega\left(\frac{n^5}{\sigma_n(H_{\alpha,\beta})}\right)$. 
$N = \Omega(n^5\sigma^{-1}_n(H_{\alpha,\beta}))$. 
%Choosing $\alpha, \beta$ corresponds to selecting a suitable parametrization, with $\varepsilon$ and $\varsigma$ reflecting its complexity and robustness.
In summary, the bound exhibits expected behavior: increasing input variance, stability, and choosing good $\alpha$ and $\beta$ reduce sample complexity, while larger model size increases it. 
\textcolor{black}{Numerical experiments confirm 
 the above discussion
about the behavior of
the estimation error, 
%that the estimation error generated by the learning algorithm follows the behavior of the proposed bound, 
see \cite[Appendix C]{racz2025finite}. }

%The bound in \eqref{eq:main1} is $O(1 / \sqrt{N})$, therefore it is statistically consistent.  In other words, the probability that the estimation error does not converge is at most $\delta$ for any $\delta > 0$, which implies that it is ultimately $0$.

\textbf{Comparison with related results.}
Our bound aligns with \cite{sarkar2019nonparametric}, 
for bounded inputs/noises, but we use a single time series with bounded signals and known system order. In contrast, \cite{sarkar2019nonparametric} does not assume system order knowledge, relying on $M$ independent time series and exponentially large Hankel matrices; their bound is $O(1/\sqrt{M})$.
Our bound is stronger than \cite{eringisRenyi}, as it depends on $\log(1/\delta)$ rather than $1/\delta$.

When applied to LTI systems ($n_Q=1$), our results yield a $O(1/\sqrt{N})$ bound, consistent with existing literature. 
%\textcolor{black}{The assumption of bounded inputs/noises 
%is often realistic.} 
\textcolor{black}{Unlike}  prior work, we do not require a stability dependent 
a lower bound on the number of Markov parameters, but we estimate at most 
$2n+1$ 
Markov parameters
independently of the degree of stability. 
This is achieved through a different proof technique that avoids linear regression and handling
%concentration inequalities for 
sub-Gaussian matrices. 
In fact, applying LTI methods to LSSs raises challenges, such as difficulties in establishing the sub-Gaussianity of the regressor due to switching, and exponentially large matrices when estimating $n_Q^T$ Markov parameters.
\textcolor{black}{Note, that identifying LSSs in general cannot be achieved by parallel LTI identification of its subsystems, see Example 2 in \cite{petreczky2010identifiability}}.
\textcolor{black}{In contrast to \cite{MASSUCCI202155}, we assume that the switching is observed, which allows us to use the Ho-Kalman algorithm, for details see \cite[Appendix D]{racz2025finite}.}

% That is, even for LTI case, our bound has certain advantages, which are gained at the expense of more restrictive assumptions. These advantages are more pronounced for LSSs than for LTIs. 
 %need a lower bound on the number of Markov parameters to 
%\textbf{Proof technique}
%The essence of the proof lies in the proof of Lemma \ref{lemma:bound_M}. 
%It is similar to \cite{oymak2021revisiting}, however we consider LSS instead of LTI systems and we omit the SVD decomposition of the Hankel matrices as we assume they are invertible. This is not restrictive and in fact achievable under mild assumptions \cite{sarkar2019nonparametric}. Additionally, as opposed to \cite{oymak2021revisiting}, our proof relies on weakly dependent processes.

%Mit irjunk meg ide? 
%\end{color}
 
%\textcolor{blue}{The proof of Theorem \ref{thm:main} requires establishing the connection between the error bound on the Markov parameters, given in Lemma \ref{lemma:bound_M}, and the estimation error between the Hankel matrices. The connection is described in the following Lemma.}

%We addressed the problem of identifying Linear Switched Systems via Markov parameters estimated from a single trajectory. Assuming that, in addition to the output and state, the switch signal is also known, we derive, to our knowledge, the first finite-sample bound for the single, long trajectory case. While our assumption about the switch signal may initially seem restrictive, it ensures statistical consistency.

\section{Conclusions and future work}
\label{sec:conc}
 We have presented a finite-sample $O(1/\sqrt{N})$ error bound for an \textcolor{black}{identification} algorithm for LSS which uses a single time series
 of length $N$. The bound exhibits the standard properties of finite-sample bounds. The proof is based on concentration inequalities for weakly dependent processes. 
 Future work is directed towards extending the results to more general switching scenarios %(e.g., Markovian switching)
 \textcolor{black}{and non i.i.d. unbounded inputs.}
 %and unbounded inputs and  noises. 

\section{Remaining proofs}
\label{sec:proofs}

\textbf{Proof of Lemma \ref{lemma:bound_H}.}
   The first equation of \eqref{eq:boundH:1} follows from Assumption \ref{ass:main} and equation \eqref{eq:Mw}, while the second one follows from the proof of \eqref{eq:boundH:3}.
     \eqref{eq:boundH:2} follows from the definition of the Frobenius norm.
    %$\norm{\widehat{H} - H}^2_F \leq \sum\limits_{w \in \mathcal{W}} \norm{\Mw - \hMw}^2_F$ by the definition of the Frobenius norm.} 
    As for \eqref{eq:boundH:3}, by \cite[Theorem 1.1]{MENG2010956} 
    %we obtain
    \\$\|\widehat{H}^{-1} - H^{-1}\|_F \leq \sqrt{2} \max\{\|\widehat{H}^{-1}\|_2^2,  \|H^{-1}\|_2^2\}\norm{\widehat{H} - H}_F$
    From \cite[(5.3.9), page 265]{golub2013matrix} follows that 
    $\|\widehat{H}^{-1}\|_2 \le \frac{1}{\sigma_n(\widehat{H})}$ and
    $\|H\|^{-1}_2 \le \frac{1}{\sigma_n(H)}$.
    From \cite[Corollary 8.6.2]{golub2013matrix}
    it follows that
    $|\sigma_n(\widehat{H})-\sigma_n(H)| \le \|\widehat{H}-H\|_2 \le \sqrt{n} \|\widehat{H}-H\|_F$\\$ \textcolor{black}{\le}  \sqrt{n \sum_{w \in \mathcal{W}} \|\Mw - \widehat{M}_w\|_F^2}$.
     That is, if for all $w \in \mathcal{W}:$ 
     $\|M_w-\widehat{M}_w\|_F < \frac{0.5\sigma_n(H)}{|\mathcal{W}|\sqrt{n}}$, then
     $|\sigma_n(\widehat{H})-\sigma_n(H)|$\\$ \textcolor{black}{\le} 0.5\sigma_n(H)$ and hence 
     $\sigma_n(\widehat{H}) \ge 0.5\sigma_n(H)$, thus
    % if $w \in \mathcal{W}$ 
     %\textcolor{blue}{$\norm{M_w-\widehat{M}_w}_F < \frac{0.5\sigma_n(H)}{|\mathcal{W}|\sqrt{n}}$}, then 
    $\|\widehat{H}^{-1} - H^{-1}\|_F
     \le \frac{\sqrt{2}}{0.25\sigma^2_n(H)} 
     \|\widehat{H}-H\|_F$.
     
\textbf{Proof of Lemma \ref{lemma:ktheta}.}
We show that $\{\bX_t\}_{t \in \mathbb{Z}}$ is a Causal Bernoulli Shift \cite[Section 3.1.4]{alquier2012}.
Let $\mathcal{X}=\mathbb{R}^{m' + n_Q}$
viewed as a Banach space with \textcolor{black}{norm $\norm{\cdot}_{1}$}.
%be a suitable Banach space \textbf{TODO: ezt majd jol kell megvalasztani}.
Let $H: \mathcal{X}^{\mathbb{N}} \rightarrow \mathbb{R}$ such that
there exist real numbers $\{a_j\}_{j=0}^{\infty}$ with the property that for any $v=\{v_j\}_{j=0}^{\infty}$,
$v'=\{v'_j\}_{j=0}^{\infty}$,
$v_j,v'_j \in \mathcal{X}$,
%$j \in \mathbb{N}$,
\begin{align}
    \abs{H(v) - H(v')} &\leq \sum\limits_{j=0}^{\infty}a_j \|v_j - v'_j\|_{1}, \quad 
    %\label{eq:cbs1} \\
    \sum\limits_{j=0}^\infty ja_j < \infty   
    \label{eq:cbs}
\end{align}
Then $\bX_t$ is a CBS, if for all $t \in \mathbb{Z}$,
    $\bX_t = H(\bxi_t, \bxi_{t-1}, \ldots)$,
where $\bxi_t$ are i.i.d.
Note, that for a fixed $l \in [1, m]$ 
\begin{align*}
    \bX_i &= \by(i)\bz_{w}^{\bu[l]}(i)  = \sum\limits_{v \in Q^{*}} M'_v \bz_{v}^\mathbf{r}(i)\bz_{w}^{\bu[l]}(i) 
    %+ D\_i\bz_{w}^{\bu[l]}(i)
   % &= \sum\limits_{v \in Q^+} M_{v} u_{i - |v| - 1} \chi_{v}(i - 1) u_{i - |w| - 1} \chi_{w}(i - 1) \\
   % &+ Du_iu_{i - |w| - 1}\chi_{w}(i - 1)
\end{align*}
We omit the dependence of $\bX_i$ on $w$ and $i$ from the notation for simplicity.
Let $\bxi(t) = [\mathbf{r}^T(t), \bchi_1(t), \ldots, \bchi_{n_Q}(t)]^T \in \mathbb{R}^{m' + n_Q}$.
Let $\Xi = \{\xi_t\}_{t \in \mathbb{N}} \in (\mathbb{R}^{m'+n_Q})^{\mathbb{N}}$ and
for all $v \in Q^{*}$
\begin{align*}
g_v(\Xi)
=g'_v(\Xi)\prod_{i=1}^{|v|} \sat_{0,1}(\xi_{|v|-i+1}[v_i + m'])
%\cdot \ldots \cdot \sat_{0,1}(\xi_{1}[v_{|v|} + 1])
\end{align*}
where $g'_v(\Xi) = \sat_{K_u}(\xi_{|v|}[1:m'])$ for $v \neq w$ and  $\sat_{K_u}(\xi_{|v|}[l])$ otherwise.
If $|v| > 0$ and $g_{\epsilon}(\Xi)=\sat_{K}(\xi_0[1:m'])$, where
$\sat_{a,b}(x) = \min\{\max\{x, a\}, b\}$ for $x\in\mathbb{R}$, $\sat_{a,b}(v)[i] = \sat_{a,b}(v[i])$ for any vector $v$
%$\sat_{a,b}(x)=\left\{\begin{array}{rl} x & x \in [a,b] \\ a & x \le a \\ b & x
%\ge b \end{array}\right.$
and $\sat_a(x)$\\$ \textcolor{black}{=} \sat_{-a, a}(x)$.
%Clearly, $\bX_i = 0$, unless $w = w_1 \ldots w_k =  q_1\ldots q_k$.
%Let $g_v(\bxi_t, \bxi_{t-1} \ldots) = \bxi_{t-|v|-1}[0]\bxi_{t-v}[v_1 + 1] \cdot \ldots \cdot \bxi_{t-1}[v_{|v|} + 1]$ as a $\mathcal{X}^{\mathbb{Z}} \rightarrow \mathbb{R}$ function. Clearly,
Then
    $\bX_i = \sum\limits_{v \in Q^*} M_v' g_v(\bXi) g_{w}(\bXi) 
    =: H^w(\bXi)$
    %+ D\bu_ig_{w}(\bxi_i, \ldots ) \\
for $\bXi = \{\textcolor{black}{\bxi(t)}\}_{t \in \mathbb{N}}$, where $H^w$ is defined as
\begin{equation}
\label{shift:eq1}
   H^w(\Xi)=
   \sum_{v \in Q^{*}} M_v' g_v(\Xi)g_w(\Xi)
\end{equation}
We show that $H^w$ satisfy
\eqref{eq:cbs} for a suitable $L_j$.

\begin{lemma}
   \label{lemma:finite1}
    Let \textcolor{black}{$Q^+ = Q^* \setminus \{\epsilon\}$}, 
    $K_{\geq j} = \sum\limits_{|v|\geq j} \|M'_v\|_2$, 
$K_{\infty} = \sum\limits_{v \in Q^+}\|M'_v\|_2$ and
        $K_j = K_u K_M$ if $j = 0$, and for $j \geq 1$: 
        $K_j = \sqrt{m'}K'_uK_{\geq j}$. Let \textcolor{black}{$\mathbf{1}$ be the indicator function},  \\
        %$L_j = K_j$ if $j \neq |w|$
        \( L_j = K_j + \textcolor{black}{ \sqrt{m'}K_u
         \mathbf{1}_{\{j \le |w|\}}(K_{\infty}
         + (K_{\infty}+K_M)\mathbf{1}_{\{j = |w|\}})} \).
    Then \eqref{shift:eq1} is absolutely convergent, $K_\infty \leq \frac{K_M}{1 - \gamma}$, 
    and the process $\bX_i$ is a CBS with $H = H^w$ and $a_j = L_j$.
    \end{lemma}
\begin{proof}
Without loss of generality, we can assume that $q_1 \ldots q_k = w = q'_1 \ldots q'_k$.
Note that for any
$\Xi \in (\mathbb{R}^{m'+n_Q})^\mathbb{N}$
and $s \in Q^{*}$
$\norm{g_s(\Xi)}_\infty \le K_u$, hence
$\abs{M_v' g_v(\Xi)g_w(\Xi)}$\\$ \textcolor{black}{\le} \sqrt{m'}\norm{M'_v}_2 K_u^2$.
From Lemma \ref{lem:stab:stat}
it follows that $\norm{M'_v}_2 \le K_M \frac{\lambda^|v|}{n_Q^{|v|}}$, implying that \eqref{shift:eq1}
can be upper bounded by a convergent geometric series, thus it is convergent. 

\textbf{Step 1: First part of  (\ref{eq:cbs})}.
    For $\Xi, \Xi' \in (\mathbb{R}^{m' + n_Q})^\mathbb{N}$
    we can upper bound $|H^{w}(\Xi) - H^{w}(\Xi')|$ by
\begin{align*}
%=
%&\Bigg|\sum\limits_{v\in Q^+}M_v\left(g_v(\textcolor{blue}{\Xi} )g_w(\textcolor{blue}{\Xi}) - g_v(\textcolor{blue}{\Xi'}) g_w(\textcolor{blue}{\Xi'})\right) \\
%&+ D\left(\textcolor{blue}{\xi_0[1:m]}g_w(\textcolor{blue}{\Xi}) - \textcolor{blue}{\xi'_0[1:m]}g_w(\textcolor{blue}{\Xi'}) \right) \Bigg | \\
\!\! \sum\limits_{v\in Q^*}\!\! \norm{M'_v}_2\norm{g_v(\Xi )g_w(\Xi) 
- g_v(\Xi') g_w(\Xi')}_2 
\end{align*}
%&+ \norm{M'_\epsilon}_2\norm{\textcolor{blue}{\xi_0[1:m']}g_w(\textcolor{blue}{\Xi_0}) - \textcolor{blue}{\xi'_0[1:m']}g_w(\textcolor{blue}{\Xi'})}_2
%Hence
%\begin{align*}
%&|D|\left|\sat_{-K,K}(u_t)g_w(\bxi_t,\ldots) - \sat_{-K,K}(u'_t)g_w(\bxi'_t, \ldots ) \right| = \\
%& \abs{D} \abs{u_tu_{t-|w|-1} - u'_tu'_{t-|w|-1}}
%\end{align*}
For the sake of readibility, let
\begin{align*}
    &\Gamma^{v,w}(\Xi, \Xi') =  g_v(\Xi)g_w(\Xi)
    - g_v(\Xi') g_w(\Xi') \\
    &\chi'_{c}(j) := \sat_{0,1}(\xi'_{j}[c+m']), \chi_{c}(j) := \sat_{0,1}(\xi_{j}[c+m']) \\
    %\begin{cases}
    %1 &  \text{if } \bq'_t = c \in Q \\
    %0 & \text{otherwise}
    %\end{cases} \\
    &\chi'_\upsilon =
    \prod_{i=1}^{k} \chi'_{\upsilon_i}(k-i), \quad \chi_\upsilon =
    \prod_{i=1}^{k} \chi_{\upsilon_i}(k - i) 
   % &\Lambda^w(\Xi_t, \Xi'_t) =  \abs{D} \abs{u_tu_{t-|w|-1}\chi_{w}(t-1) - u'_tu'_{t-|w|-1}\chi'_w(t-1)} \\
\end{align*}
for any $\upsilon=\upsilon_1\cdots \upsilon_k \in Q^{+}$,
$\upsilon_i \in Q$, $i\in [1,k]$.
\\Let $\bar{u}_k=\sat_{K_u}(
\xi_k[1:m'])$, 
$\bar{u}'_k=\sat_{K_u}(\xi_k'[1:m'])$,
\\$s=\sat_{K_u}(\xi_{\textcolor{black}{|w|}}[l])$ and $s'=\sat_{K_u}(\xi'_{\textcolor{black}{|w|}}[l])$ for all $k$. Then
\begin{align*}
    &\|\Gamma^{v,w}(\Xi, \Xi')\|_2 =
   % &\begin{cases}
   %     u_{t - |v| - 1} u_{t - |w| - 1} & \text{if } \chi'_{v}(t-1) = 0,  \\
   %     u'_{t - |v| - 1} u'_{t - |w| - 1} & \text{if } \chi_{v}(t-1) = 0,  \\
     \|\bar{u}_{\textcolor{black}{|v|}} s \chi_{v}\chi_{w}- 
      \bar{u}'_{\textcolor{black}{|v|}}s'
\chi'_v\chi'_w\|_2.
        %& \begin{aligned}[t]
        %    &\text{if } \chi'_{v}(t-1) \neq 0 \\
        %    &\text{and } \chi_v(t-1) \neq 0,
       % \end{aligned} \\
       % 0 & %\text{otherwise}.
    %\end{cases}
\end{align*}
%By using 
%that $|\chi_w\chi_v| \le 1$, 
This yields to
\begin{align*}
     \|\Gamma^{v,w}(\Xi, \Xi')\|_2 &\leq 
     \|\bar{u}_{\textcolor{black}{|v|}} s \chi_{v}\chi_{w}-
    \bar{u}_{\textcolor{black}{|v|}}^{'} s'\chi_{v}\chi_{w}\|_2  \\
    & +\|\bar{u}_{\textcolor{black}{|v|}}^{'} s' \chi_{v}\chi_{w}-
    \bar{u}'_{\textcolor{black}{|v|}}s'\chi'_{v}\chi'_{w}\|_2
\end{align*}
Now 
$\norm{\sat_{K_u}(x)}_2 \leq \sqrt{m'}K_u$ and $\norm{\sat_{0,1}(x)}_2 \leq \sqrt{m'}$. %Hence
\textcolor{black}{As} $|\chi_w\chi_v| \le 1$
and $\max\{\norm{\bar{u}_k}_2, \norm{\bar{u}'_k}_2\} \le \sqrt{m'}K_u$,
%$\norm{\bar{u}'_k}_2 \le \sqrt{m}K_u$.
thus
\begin{align*}
%\Gamma^{v, w}(\Xi, \Xi') \leq
& \|\bar{u}_{\textcolor{black}{|v|}} s \chi_{v}\chi_{w}-
    \bar{u}'_{\textcolor{black}{|v|}}s'\chi_{v}\chi_{w}\|_2  \\
 & \textcolor{black}{\leq} K_u \|\bar{u}_{|v|}  - \bar{u}'_{|v|}\|_1 + \sqrt{m'}K_u|s  - s'|
%& |\bar{u}_{|v| + 1} \bar{u}_{|w| + 1} \chi_{v}\chi_{w}-
%    \bar{u}'_{|v| + 1}\bar{u}'_{|w| + 1}\chi_{v}\chi_{w}|  \\
%  & \le  |\bar{u}_{|v| + 1}\bar{u}_{|w| + 1}-
%    \bar{u}'_{|v| + 1}\bar{u}'_{|w| + 1}| \\
%& \le
%    |\bar{u}_{|v| + 1} \bar{u}_{|w| + 1}- \bar{u}'_{|v| + 1}\bar{u}_{|w| + 1}| \\
% & + | \bar{u}'_{|v| + 1}\bar{u}_{|w| + 1} -\bar{u}'_{|v| + 1}\bar{u}'_{|w| + 1}| \\
%&  \le K_u \abs{\bar{u}_{|v|+1}  - \bar{u}'_{|v|+1}} + K_u\abs{\bar{u}_{|w|+1}  - \bar{u}'_{|w|+1}}
\end{align*}
Moreover, notice that
\begin{align*}
& \|\bar{u}'_{\textcolor{black}{|v|}} s' \chi_{v}\chi_{w}-
    \bar{u}'_{\textcolor{black}{|v|}}s'\chi_{v}'\chi_{w}'\|_2
%\\
\textcolor{black}{\le}  \|\bar{u}'_{\textcolor{black}{|v|}} s'\|_2
|\chi_{v}\chi_{w}-
\chi_{v}'\chi_{w}'|\!\!   \\
& \leq  \textcolor{black}{\sqrt{m'}K_u^2(
|\chi_{v}-\chi_v^{'}|+|\chi_{w}-\chi_w^{'}|)}  \leq  \textcolor{black}{\sqrt{m'}K_u^2\cdot(} \\
&\textcolor{black}{ \sum_{j=1}^{|v|}\sum_{c \in Q}
 |\chi_c(j)-\chi'_c(j)|  + \sum_{j=1}^{|w|}\sum_{c \in Q}
 |\chi_c(j)-\chi'_c(j)|)}
 %\sqrt{m} K_u^2 \abs{\chi_{v}\chi_{w}-
%\chi_{v}\chi_{w}'} \le \\
%& \textcolor{black}{\sqrt{m}K_u^2 \text{ez itt nem } & K^3?} 
 %\textcolor{black}{\leq} & \sqrt{m'}K_u^2 \sum_{j=1}^{|w|+|v|}\sum_{c \in Q}
 % |\chi_c(j)-\chi'_c(j)|
\end{align*}
%Since $\sat_{a,b}$ is $1$-Lipschitz,
As $\norm{\bar{u}_k -\bar{u}'_k}_1 \le \norm{u_k - u'_k}_1= \norm{\xi_k[1:m']-\xi'_k[1:m']}_1$,
$\abs{\chi_c(j)-\chi^{'}_c(j)} \le \abs{\xi_{j}[c+m']-\xi'_{j}[c+m']}$ for all
\\$1 \leq j \leq |v| + |w|$, because $\sat_{a,b}$ is $1$-Lipschitz, implying
\begin{align*}
    &\|\Gamma^{v,w}(\Xi, \Xi')\|_2  \leq K_u\|\xi_{\textcolor{black}{|v|}}[1:m']  - \xi'_{\textcolor{black}{|v|}}[1:m']\|_2 \\
    &+ \sqrt{m'}K_u|\xi_{\textcolor{black}{|w|}}[l]  - \xi'_{\textcolor{black}{|w|}}[l]| \\
    &+ \sqrt{m'}K_u^2 \textcolor{black}{\sum\limits_{j=1}^{\max\{|v|,|w|\}} c_j \sum\limits_{c \in Q}|\xi_{j}[c+m'] - \xi'_{j}[c+m']|}
    %& \textcolor{black}{\sqrt{m'}K_u^2
    %\sum\limits_{j=1}^{|w|}\sum\limits_{c \in Q}|\xi_{j}[c+m'] - \xi'_{j}[c+m']|}
\end{align*}
\textcolor{black}{where $c_j=2$ for $j \le |w|$ and $c_j=1$ for $j > |w|$.}
Note that 
\begin{align*}
    \textcolor{black}{\|\Gamma^{\epsilon, w}(\Xi, \Xi')\|_2} &\leq \textcolor{black}{K_u} \|\xi_0[1:m'] - \xi'_0[1:m']\|_2 \\
    &+ \sqrt{m'}K_u |\xi_{\textcolor{black}{|w|}}[l] - \xi'_{\textcolor{black}{|w|}}[l]|
\end{align*}
We have
\begin{align*}
&|H^w(\Xi) - H^w(\Xi')| \leq 
\sum\limits_{v \in Q^*} \|M'_v\|_2 \|\Gamma^{v,w}(\Xi, \Xi')\|_2  \\
%&\sum\limits_{v \in Q^+} K_u\|M_v\|_2\|\xi_{|v|+1}[1:m] - \xi'_{|v|+1}[1:m]\|_2  \\
%&+ |\xi_{|w|+1}[l] - \xi'_{|w|+1}[l]|\sqrt{m}K_u \Big( \sum\limits_{v \in Q^+} |M_v| \Big)  \\
%&+ \textcolor{blue}{\sqrt{m}K_u^2} \sum\limits_{v \in Q^+}\sum\limits_{j=1}^{|v| + |w|}\sum\limits_{c \in Q} \|M_v\|_2 |\chi_c(j) - \chi'_{c}(j)| = \\
& \textcolor{black}{\leq} \sum\limits_{j=\textcolor{black}{1}}^\infty \Big[ \big[\|\xi_{\textcolor{black}{j}}[1:m'] - \xi'_{\textcolor{black}{j}}[1:m']\|_2 K_u \sum\limits_{|v| = j}\|M'_v\|_2  \\
 &+ K_u\sum\limits_{c \in Q} |\chi_c(j) - \chi'_c(j)|\big] (\sqrt{m'}K_u (\sum\limits_{|v| \geq j} \|M'_v\|_2 \\
 & + \textcolor{black}{\chi(j \le |w|) 
 \sum_{v \in Q^{+}}  \norm{M'_v}_2)}
 \Big ] \\
 %& 
 %\textcolor{black}{
 % +\sqrt{m'}K_u^2 (\sum_{j=1}^{|w|}\sum\limits_{c \in Q} |\chi_c(j) - \chi'_c(j)|) (\sum_{v \in Q^{+}}  \norm{M'_v}_2)
 %}
 %\\
 &+ |\xi_{\textcolor{black}{|w|}}[l] - \xi'_{\textcolor{black}{|w|}}[l]|\sqrt{m'}K_u \Big( \textcolor{black}{\norm{M'_\epsilon}_2 +} \sum\limits_{v \in \textcolor{black}{Q^{+}}} \norm{M'_v}_2 \Big) 
% & \sum\limits_{j=1}^\infty L_j \norm{\xi_{t-j} - \xi'_{t-j}}_1 + \Lambda^w(\xi_t, \xi'_t)
\end{align*}

\textcolor{black}{By using that by Lemma \ref{lem:stab:stat} $\norm{M'_\epsilon}_2 \leq K_M$ and}
the definition of $L_j$ and by $\|\cdot\|_2 \leq \|\cdot\|_1$ we have
$|H^w(\Xi) - H^w(\Xi')| \leq
  \sum\limits_{j=0}^\infty L_j \|\xi_{j} - \xi'_{j}\|_1$.

\textbf{Step 2: Second part of (\ref{eq:cbs})}.
Note, that $K_{\geq j} \leq K_\infty$ and  
\begin{align*}
    K_\infty = \sum_{j=0}^\infty \sum_{|v|=j}\|M'_v\|_2 \leq \sum\limits_{j=0}^\infty K_M n_Q^j \frac{\gamma^j}{n_Q^j} = \frac{K_M}{1 - \gamma}
    \end{align*}
by Lemma \ref{lem:stab:stat}.
As for the second equation in  \eqref{eq:cbs},
\begin{align}
    & \sum\limits_{j=0}^{\infty} j L_j = \sum\limits_{\textcolor{black}{j=1}}^\infty \big[ j \sqrt{m'}K_u \sum\limits_{|v| \geq j} \|M'_v\|_2 \big] \nonumber \\
    &+ %\textcolor{black}{\sum_{j=1}^{|w|} j} 
    \sqrt{m'}K_u (
    %K_{\geq |w|}+ 
    \textcolor{black}{\sum_{j=1}^{|w|} j K_{\infty}}
    +\textcolor{black}{|w|}(K_\infty + K_M))
    \label{sumjaj:1}
\end{align}
%What's left to show is
%\begin{align*}
%    \sum\limits_{j=1, j\neq|w|}^\infty j \sum\limits_{|v|\geq j}\norm{M_v}_2 < + \infty.
%\end{align*}
%
Again, due to Lemma \ref{lem:stab:stat}
and 
\textcolor{black}{the definition of $K_{\geq j},$
%$K_{\ge j}=\sum\limits_{|v|\geq j}\|M'_v\|_2$
%\le \sum\limits_{k=j}^\infty n_Q^k \frac{\gamma^k}{n_Q^k}$,
}
\begin{align}
    &\sum\limits_{\textcolor{black}{j=1}}^\infty j K_{\textcolor{black}{\geq j}} 
    \leq \sum\limits_{j=1}^\infty j K_M \sum\limits_{k=j}^\infty n_Q^k \frac{\gamma^k}{n_Q^k}
    %\nonumber \\
    %\\
   %& = K_M \sum\limits_{j=1}^\infty j \frac{\gamma^j}{1-\gamma} =
   %& 
   =\frac{K_M \gamma}{(1-\gamma)^3} 
    \label{sumjaj:2}
   %\leq +\infty
\end{align}
\end{proof}
%We are ready to prove Lemma \ref{lemma:ktheta}.
%\textbf{Proof of Lemma \ref{lemma:ktheta}}
\begin{proof}[Proof of Lemma \ref{lemma:ktheta}]
%\label{proof:lemma:ktheta}
From \cite[Proposition 4.2]{alquier2012} 
it follows that 
that for a CBS, 
if $\norm{\bxi_0}_1 \leq b$, then $\theta_{\infty, N}(1) \leq 2b \sum\limits_{j=0}^\infty ja_j$.
	From the proof of Lemma \ref{lemma:finite1} we have that $\norm{\bxi_0}_1 \leq \textcolor{black}{m'K_u + n_Q}$.
Then the statement follows
from \eqref{sumjaj:1} and
\eqref{sumjaj:2} and
the definition of $K_{\theta}$,
and using $|w| \leq 2n + 1$ due to the definition of $\mathcal{W}$.
%Therefore, by Step 2. of the aforementioned proof, we have
%\begin{align*}
%    &2b \sum\limits_{j=0}^\infty j L_j \leq |w|\textcolor{blue}{\sqrt{m}}K'_u(K_{\geq |w|} + K_\infty + K_M) \\
%    &+ 2(\textcolor{blue}{\sqrt{m}}K_u + n_Q) + \frac{\textcolor{blue}{\sqrt{m}}K'_uK_M\gamma}{(1-\gamma)^3}  \leq
%    2(\textcolor{blue}{\sqrt{m}}K_u + n_Q)\\
%    &+ \textcolor{blue}{\sqrt{m}}K'_u\left(|w|K_M + \frac{K_M (|w|(\gamma^{|w|} + 1)(1-\gamma)^2 + \gamma)}{(1-\gamma)^3} \right)  
%\end{align*}
%where the last term is at most $K_\theta$ (defined in Lemma \ref{lemma:bound_M}) and 
\end{proof}

\bibliographystyle{plain}
\bibliography{references_final}

%\addtolength{\textheight}{-12cm}   % This command serves to balance the column lengths
%                                  % on the last page of the document manually. It shortens
%                                  % the textheight of the last page by a suitable amount.
%                                  % This command does not take effect until the next page
%                                  % so it should come on the page before the last. Make
%                                  % sure that you do not shorten the textheight too much.
%

\clearpage
\onecolumn
\appendix

\subsection{Proof of Lemma \ref{lem:stab:stat}}

Let $P$ and $\gamma$ \textcolor{black}{be} as in Assumption \ref{ass:main}
and $K_M=\max\{\|C\|_2 \max_{q \in Q} \|B_q'\|_2 \frac{n_Q}{\gamma} \textcolor{black}{\sqrt{\frac{\lambda_1(P)}{\lambda_n(P)}}}, \|\textcolor{black}{[D,1]}\|_2\}$.

Then $A_w^TPA_w \preceq \textcolor{black}{\frac{\gamma^{2|w|}}{n_Q^{2|w|}}} P$ \textcolor{black}{holds and implies} $\|A_w\|_2 \le \frac{\gamma^{|w|}}{n_Q^{|w|}} \textcolor{black}{\sqrt{\frac{\lambda_{1}(P)}{\lambda_{n}(P)}}}$.
%where $\lambda_{n}(P)$ and $\lambda_{1}(P)$
%are the minimal and maximal eigenvalues of $P$
%respectively.
In particular, \textcolor{black}{it follows that}
$\|M'_{iv}\|_2 \le \|CA_vB'_i\|_2$\\$\le \|C\|_2 \|A_v\|_2 \|B_i'\|_2 \le \|C\|_2 \|B_i'\| \frac{\textcolor{black}{\gamma}^{|v|}}{n_Q^{|v|}} \\\textcolor{black}{\sqrt{\frac{\textcolor{black}{\lambda_{1}}(P)}{\lambda_{n}(P)}}} \le K_M \frac{\gamma^{|v|+1}}{n_Q^{|v|+1}}$, \textcolor{black}{where we used the submultiplicativity of the induced matrix norm. }This yields statement a).

For b) note, that $\sum_{q \in Q} p_q A_q^TPA_q \preceq \sum_{q \in Q} A_q^TPA_q \preceq \textcolor{black}{\gamma^2} P$ \textcolor{black}{due to Assumption \ref{ass:main}}. Then \cite[Theorem 3.9]{CostaBook} implies
 $\sum_{q \in Q} p_q A_q \otimes A_q$
is Schur. From \cite[Lemma 1]{PetreczkyBilinear} and \cite[Example 2]{PetreczkyBilinear} it follows that
there exists a unique process
$\bx$ and $\by$ such that
\eqref{eq:M} holds and the infinite sums converge in the mean-square sense.
Stationarity of
$\bx_t$ follows from
it being a mean-squared limit
$\mathbf{S}_I(t)=\!\!\! \!\! \sum\limits_{v \in Q^{*}, |v| < I,q \in Q}\!\!\!\! A_vB'_q \bz_{qv}^{\mathbf{r}}(t)$
as $I \rightarrow \infty$. As mean square convergence implies convergence in distribution, it follows
that the joint distribution of $\bx_{t_1+t},\ldots,\bx_{t_k+t}$ is the limit of the joint distribution of
$\mathbf{S}_I(t_1+t),\ldots,\mathbf{S}_I(t_k+t)$ as $I \rightarrow \infty$,
and the latter distribution function does not depend on $t$, as $\bz_{qv}^{\mathbf{r}}(t)$ is stationary, and hence its linear combination is stationary, too.
If $\bx$ is stationary, then so is $\by$ which arises as a linear combination of $\bx$ and $\mathbf{r}$.
Finally, \eqref{eq:Mw} follows from
\eqref{eq:M} by the equality
$\mathbb{E}[\by(t)(\bz_w^{\bu}(t))^T]=\sum_{v \in Q^{*}} M'_v \mathbb{E}[\bz_{qv}^{\mathbf{r}}(t)(\bz_w^{\bu}(t))^T]$ \textcolor{black}{due to the linearity of the expectation}, by the fact that
%+ CA_vE[\bz^{\bw}_{qv}(t)(\bz_w^{\bu}(t))^T]+DE[\bv(t)(\bz_w^{\bu}(t))^T]$, 
$\mathbb{E}[\bz_v^{\mathbf{r}}(t)(\bz_w^{\bu}(t))^T]=0$ as long as $v \ne w$ \textcolor{black}{due to independence},
and by
$\mathbb{E}[\bz_w^{\mathbf{r}}(t)(\bz_w^{\bu}(t))^T]=[\Sigma_u^T p_w, 0]^T$ \textcolor{black}{due to Assumption \ref{ass:main}}. This concludes b).

As for c), $\norm{\bz_v^{\mathbf{r}}(t)}_{\infty} \le 
\norm{\mathbf{r}(t-|v|)}_{\infty} |\bchi_v(t)| \le K_u$ \textcolor{black}{by the definition of $\mathbf{z}$ and $\bchi$} as
$|\bchi_v(t)| \le 1$. Then $|\by(t)| \le \sum_{v \in Q^{*}} \|M'_v| \|_2 \|\bz_v^{\bu}(t)|_2 \le \sum_{k=0}^{\infty} \sum_{|v|=k} K_M m' K_u \frac{\gamma^k}{n_Q^k}=K_y$ concludes c).
%=\sum_{k=0}^{\infty} mK_MK_u \gamma^k=K_y$.

\begin{color}{black}
\subsection{Multi-output case}
The extension for the multi-output case is based on the same technique as the proof of Lemma 3.1. Namely, in the aforementioned proof, we first prove eq. (10) for a fixed input index $l$ and a fixed $w \in \mathcal{W}$ and take the union bound over the set of input indices and then over the set of nice selections $\mathcal{W}$. The exact same technique can be applied to the case of multiple outputs.

Consider an output dimension $r$. Now the Markov parameters are matrix valued and in the nice selection (Section II, par. Reduced basis Ho-Kalman algorithm), $\alpha$ has to be indexed by the output dimensions as well, i.e. $\alpha \subseteq Q^* \times \{1,\ldots, r\}$, see e.g. \cite{MertBastug:TAC}. Then, for a fixed $w \in \mathcal{W}$, fixed input dimension $l$ and fixed output dimension $k$, eq. (10) has the form
\begin{align*}
\label{eq:wl}
    \mathbb{P}\big[|\tMw[k,l] - \htMw[k,l]| > \widetilde{K}(\delta, N)\big] \leq
    \delta r^{-1}m ^{-1}|\mathcal{W}|^{-1}.
\end{align*}

The rest of the proof of Lemma 3.1. works for this modified version of eq. (10) and we can analogously take the union bound over the set of output indices as well. The rest of the Lemmas and Theorems that build on Lemma 3.1. hold without any additional modification (except that the constants change slightly).

\end{color}

\begin{color}{black}
\subsection{Numerical example}
	In order to illustrate the conclusions of the bound, we simulated an LSS of the form \eqref{eq:1} with $Q=\{1,2\}$,
	$\mathbf{q}(t)$ is an i.i.d. processes with $\mathbb{P}[\mathbf{q}(t)=q]=1/2$ for any $q=1,2$, with the input
	$\bu,\bw,\bv$ are mutually independent white noise processes sampled from the uniform distribution on $[-K_u,K_u]$.
	Furthermore, $A_1=\begin{bmatrix} 0 & 1 \\ 0 & a \end{bmatrix}$, $A_2=\begin{bmatrix} 0 & 0 \\ 0 & a \end{bmatrix}$,
	$B_1=\begin{bmatrix} 0 & 0 \end{bmatrix}$, $B_2=\begin{bmatrix} 0 & 1 \end{bmatrix}$, $C=\begin{bmatrix} 1 & 0 \end{bmatrix}$,
	$D=0$. 
	It is clear that the LSS above is minimal, i.e., it is reachable and observable, for $\bw=0$ and $\bv=0$. 
	Indeed, the matrices $R=\begin{bmatrix} B_2 & A_1B_2 \end{bmatrix}$, $O=\begin{bmatrix} C \\ CA_1 \end{bmatrix}$
	are full row and column rank respectively, and they are submatrices of the $2$-step extended reachability
	and observability matrices of the LSS respectively. Hence by \cite{Petreczky2015} the LSS is reachable, observable and
	minimal.

	Note that the transfer functions of the LTI subsystems associated with both discrete modes are zero and hence do not depend on $a$. 
	Hence, it is impossible to identify the parameter $a$ by simply running a classical LTI system identification scheme 
	in parallel for both LTI systems. At the same time, the knowledge of the Markov-parameter $CA_1A_2B_2=a$ allows us to
	identify $a$. This Markov-parameter corresponds to the output at time $t=3$ for input $\bu$ such that $\bu(0)=1$, $\bu(t)=0$, $t >0$,
	and switching signal $\mathbf{q}(0)=2$, $\mathbf{q}(1)=2$, $\mathbf{q}(2)=1$ and no noise. 

	In particular, for $\alpha=\{\epsilon,1\}$ and $\beta=\{(\epsilon,2,1),(1,2,1)\}$ the Hankel matrix 
	$H_{\alpha,\beta}=\begin{bmatrix} C_1B_2 & C_1A_1B_2 \\ C_1A_1B_2 & C_1A_1A_1B_2 \end{bmatrix}=\begin{bmatrix} 0 & 1 \\ 1 & a \end{bmatrix}$ is invertible.

	For $\gamma \in \{0.1,0.4,0.6\}$, we set $a=0.9\gamma/2$. By solving the LMI $A_q^TPA_a-\gamma^2/4P < 0$, $q=1,2$, $P > 0$, 
	we verified numerically that the LSS satisfies Assumption 2.1, e). 
	We sampled $\bw,\bv$ from the uniform distribution $[-K_u,K_u]$ and we sampled $\bu$ from the uniform 
	distribution $[-K_{u,inp},K_{u,inp}]$ for $K_u \in \{1,20,30\}$ and $K_{u,inp} \in \{0.5,0.8\}$. 
	We used the simulated output as data for the Ho-Kalman-based identification algorithm. 
	We computed the maximal estimation error  of the matrices  $\mathbf{EstErr}$ from Theorem \ref{thm:main}, 
	i.e., the maximum of the Frobenius norm of 
	the difference between matrices $\bAhat_q,\bBhat_q, \bChat$ and the matrices $\bar{A}_q,\bar{B}_q,\bar{C}$,
	where the latter matrices are the result of applying the the Ho-Kalman realization algorithm to the true Hankel-matrices
	$H_{\alpha,\beta}$, $H_{\alpha,q,beta}$, $H_{\alpha,q}$, $H_{\beta}$, $q \in Q$. 

	The results for various choices of $\gamma$, $K_u$ and $K_{u,inp}$ are presented in Fig. \ref{fig1}, \ref{fig2} and \ref{fig3}.
	The estimation error  $\mathbf{EstErr}$ behaves as predicted by the bound of Theorem \ref{thm:main}. 
	That is, the estimation error increases with $\gamma$ and $K_u$ and decreases with the variance $4K_{u,inp}^2/12$ of $\bu$.
    Note that the effect of the Lyapunov exponent $\gamma$ is relatively minor on the estimation error: this is consistent with the behavior of the constant $K_0$, which 
    converges to a constant as $\gamma \rightarrow 0$.

	\begin{figure}[ht!]
	\caption{Effect of stability on the parameter estimation error \label{fig1}}
	\centering
	\includegraphics[scale=0.8]{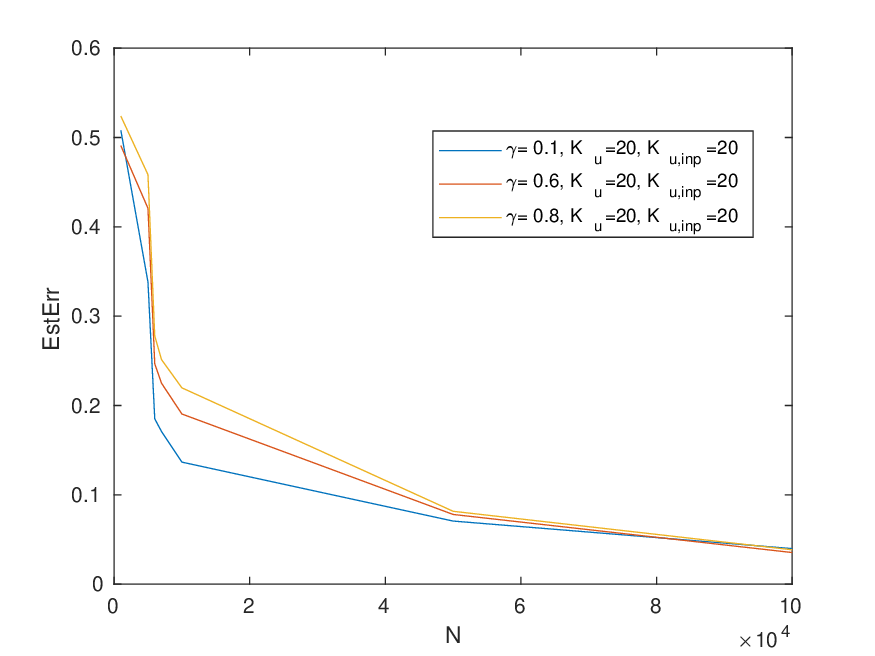}
        \end{figure}
    \vspace{-10pt}
	\begin{figure}[ht!]
		\caption{Effect of the variance $4K_{u,inp}^2/12$ of $u$ on the estimation error: smaller values of $K_{u,inp}$ correspond to smaller variance. \label{fig2}}
	\begin{center}	\includegraphics[scale=0.8]{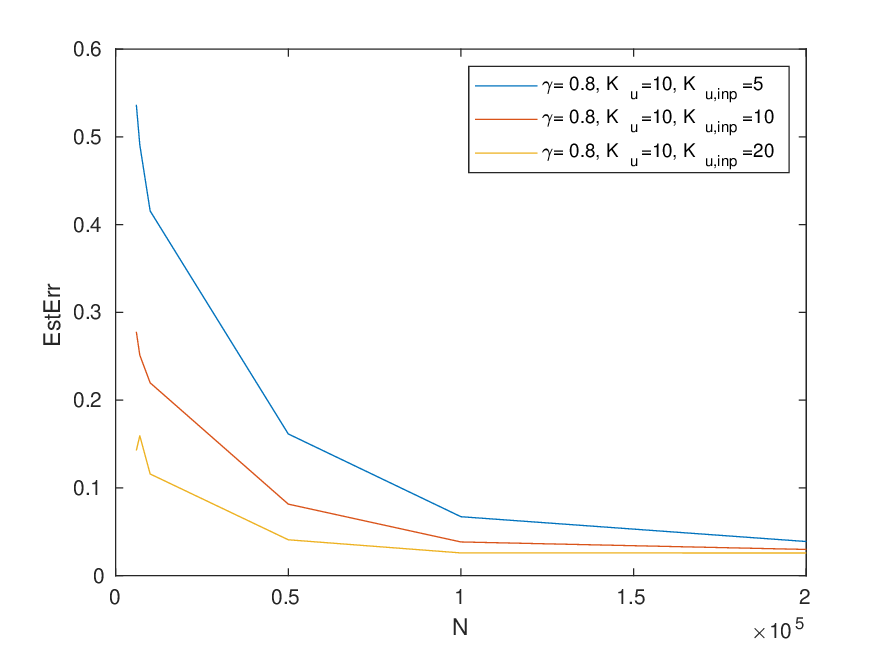}
	\end{center}
	\end{figure}
\begin{figure}[H]
		\caption{Effect of the magnitude $K_u$ of the input and noise  on the estimation error \label{fig3}}
	\begin{center}	\includegraphics[scale=0.8]{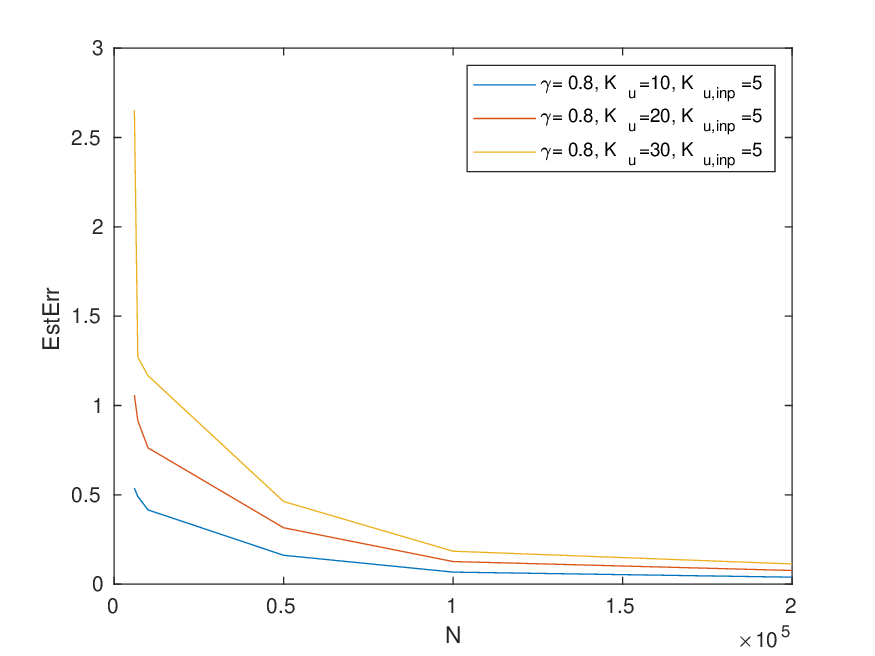}
	\end{center}
	\end{figure}
\end{color}	

The implementation of the above example is available at \href{https://gitlab.com/mpetrec/pac-lss}{https://gitlab.com/mpetrec/pac-lss}.

\begin{color}{black}
    \subsection{Effect of observed switching signal on Theorem \ref{thm:main}}

    As we stated in Section \ref{sec:main}, the bounding term in Theorem \ref{thm:main} behaves as expected, namely increasing input variance, stability and choosing a good selection of $\alpha$ and $\beta$ reduce the sample complexity, while larger model size increases it.

    As for the effect of known switching signals on Theorem \ref{thm:main}, to the best of our knowledge, there is no equivalent of the Ho-Kalman realization theory in the literature for switched systems in state-space form with unobserved switching. Consequently, it is unclear how to derive a Ho-Kalman based identification algorithm for such systems and therefore the effect of known switching compared to unknown switching is also unclear to analyze in a general manner.

    Furthermore, the existing work on switched systems with unobserved switching usually assume that the system is in an autoregressive form, i.e. its state can be observed. An example for this is in \cite{MASSUCCI202155}, where an algorithm agnostic PAC bound was derived depending on certain $\beta$-mixing coefficients which are usually difficult to estimate. What is more, the authors introduced an additional assumption and proved \cite[Theorem 2]{MASSUCCI202155} in order to "\textit{bypass the issue of estimating the mixing coefficient}" (see the discussion after \cite[Theorem 2]{MASSUCCI202155}).
    
    In contrast, our bound depends on different mixing coefficients $\theta$ arising from the theory of weakly dependent processes. As opposed to the $\beta$-mixing coefficient, the mixing coefficient $\theta$ can be estimated under some mild assumptions on the true, unknown system. Specifically, estimating $\theta$ requires an upper bound on the largest Lyapunov exponent of the true system.
    
More concretely, Lemma \ref{lemma:ktheta} estimates $\theta$ by showing that the Markov parameters can be generated by an underlying Bernoulli shift, for which the known switching signals are essential. Note, that changing Assumption \ref{ass:main} to Markovian switching could already raise difficulties when it comes to establish the Bernoulli shift structure. Estimating similar mixing coefficients in the unobserved case and the general effect of known switching signals on these estimations requires future research.
    
\end{color}
\end{document}